\documentclass[conference]{IEEEtran}
\usepackage{times}
\usepackage{graphicx} 
\usepackage{epstopdf}
\usepackage{mathptmx} 
\usepackage{times} 
\usepackage{amsmath} 
\usepackage{amssymb}  
\usepackage{caption}
\usepackage{subcaption}
\usepackage{grffile}
\usepackage{lipsum}                     
\usepackage{xargs}                      
\usepackage[pdftex,dvipsnames]{xcolor}  

\usepackage[numbers]{natbib}
\usepackage{hyperref}
\usepackage[colorinlistoftodos,prependcaption,textsize=tiny, disable]{todonotes}
\newcommandx{\unsure}[2][1=]{\todo[linecolor=red,backgroundcolor=red!25,bordercolor=red,#1]{#2}}
\newcommandx{\change}[2][1=]{\todo[linecolor=blue,backgroundcolor=blue!25,bordercolor=blue,#1]{#2}}
\newcommandx{\info}[2][1=]{\todo[linecolor=OliveGreen,backgroundcolor=OliveGreen!25,bordercolor=OliveGreen,#1]{#2}}
\newcommandx{\improvement}[2][1=]{\todo[linecolor=Plum,backgroundcolor=Plum!25,bordercolor=Plum,#1]{#2}}
\newcommandx{\thiswillnotshow}[2][1=]{\todo[disable,#1]{#2}}
\newtheorem{theorem}{Theorem}
\DeclareMathOperator{\Tr}{Tr}
\newtheorem{lemma}{Lemma}
\newtheorem{corollary}{Corollary}
\newcommand*{\mylen}{1.6in}
\newcommand*{\mylenlong}{2.8in}
\newcommand*{\mylenmedium}{1.8in}

\newcommand*{\figbarlen}{1.4in}
\newcommand*{\mylenshort}{1.2in}

\begin{document}

\title{A Fast Stochastic Contact Model for Planar Pushing and Grasping:
  Theory and Experimental Validation}
\author{Jiaji Zhou,  J. Andrew Bagnell and Matthew T. Mason  \\
 The Robotics Institute, Carnegie Mellon University \\
 \{jiajiz, dbagnell, matt.mason\}@cs.cmu.edu}

\maketitle

\begin{abstract}
Based on the convex force-motion polynomial model for quasi-static sliding, we derive the kinematic contact model to determine the contact modes and instantaneous object motion on a supporting surface given a position controlled
manipulator. The inherently stochastic object-to-surface friction
distribution is modelled by sampling physically consistent parameters from appropriate distributions, with only one parameter to control the amount of noise. 
Thanks to the high fidelity and smoothness of convex
polynomial models, the mechanics of patch contact is captured while
being computationally efficient without mode selection at support points.   
 The motion equations for both single and multiple frictional
contacts are given. Simulation based on the model is validated with robotic pushing
and grasping experiments. \footnote{Our open-source simulation
  software and data are available at: \scriptsize{\url{https://github.com/robinzhoucmu/Pushing}}} 
\end{abstract}

\section{Introduction}
Uncertainty from robot perception and motion inaccuracy is ubiquitous. Planning and
control without explicit reasoning about uncertainty can lead to undesirable
results. For example, grasp planning \cite{miller2003automatic, ferrari1992planning} is
often prone to uncertainy: the object moves while the fingers close
and ends up in a final relative pose that differs from planned. Consider the process of closing a
parallel jaw gripper shown in Fig. \ref{fig:grasp_failure}, the object will slide when the
first finger engages contact and pushes the object before the other
one touches the object. If the object does not end up slipping out, it
can be jammed at an undesired position or grasped at an unexpected
position. A high fidelity and easily identifiable model with minimum
adjustable parameters capturing all these possible outcomes would enable synthesis of robust manipulation strategy.

Although we can reduce uncertainty by carefully controlling the
robot's environment, as in most factory automation scenarios, such approach is both expensive
and inflexible. Effective robotic manipulation requires an understanding of the
underlying physical processes. Mason \cite{Mason1986a} explored using pushing as a
sensorless mechanical funnel to reduce uncertainty.  
Whitney \cite{Whitney1983b} analyzed the mechanics of wedging and jamming during
peg-in-hole insertion and designed the Remote Center Compliance device
that significantly increases the success of the operation under
motion uncertainty. 
With a well defined generalized damper model, Lozano-Perez et al.
\cite{Lozano-Perez1984a} and Erdmann \cite{Erdmann1986} developed
strategies to chain a sequence of operations, each with a certain funnel, to guarantee operation
success despite uncertainty. These successes stem from robustness analysis using simple physics models.
 
A large class of manipulation problems involve finite planar sliding
motion. In this paper, we propose a quasi-static kinematic contact
model for such a class. We model the inherent stochasticity in frictional sliding by
sampling the physics parameters from proper
distributions. We validate the model by comparing simulation with
large scale experimental data on robotic pushing and grasping
tasks. The model serves as good basis for both open loop planning and
feedback control. 

\begin{figure}[!t]
\centering
\begin{subfigure}[t]{\figbarlen}
\centering
\includegraphics[width=\figbarlen]{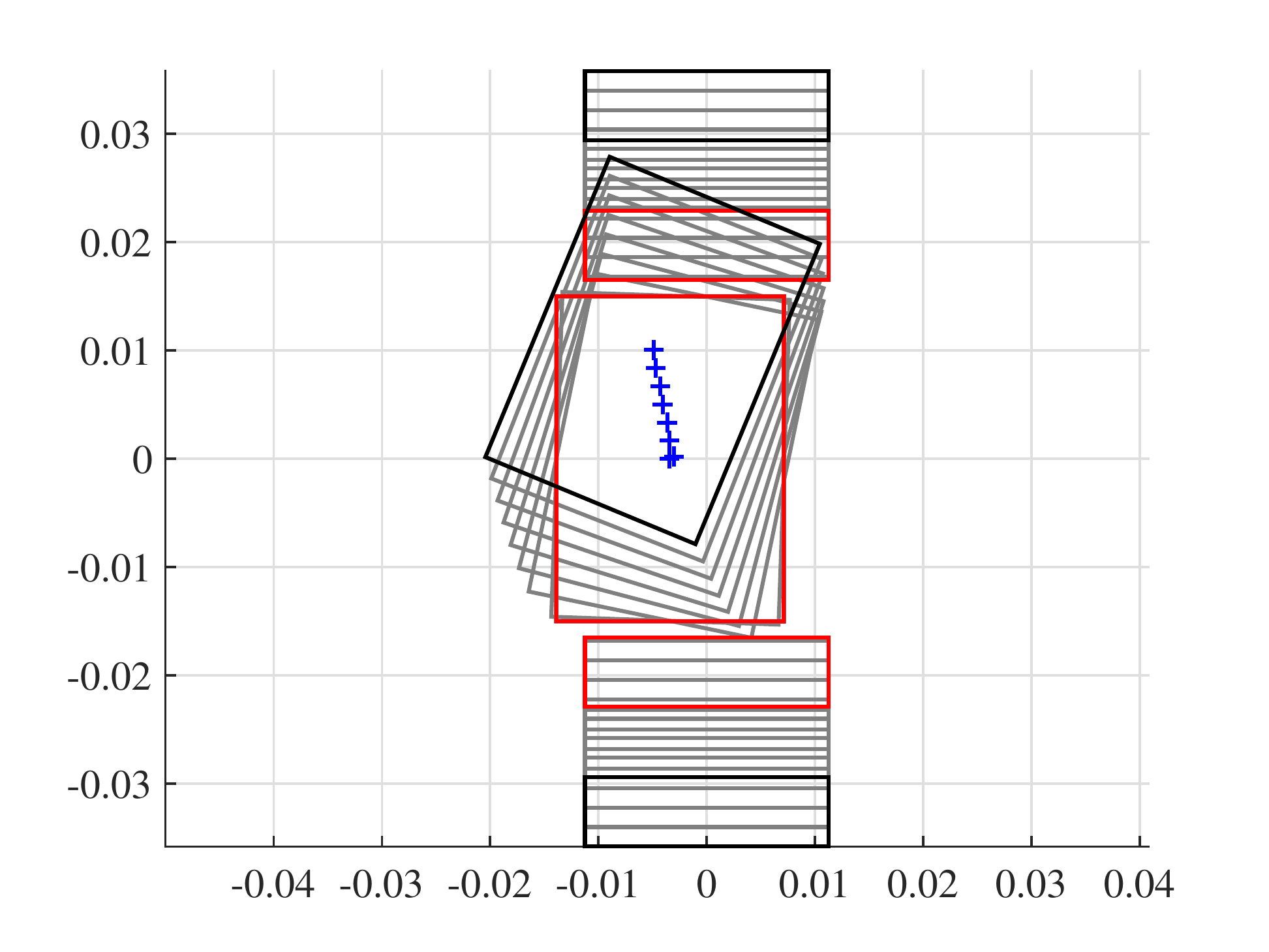} 
\caption{Grasped with offset.\\}
\label{fig:offset_grasp1}
\end{subfigure}
~
\begin{subfigure}[t]{\figbarlen}
\centering
\includegraphics[width=\figbarlen]{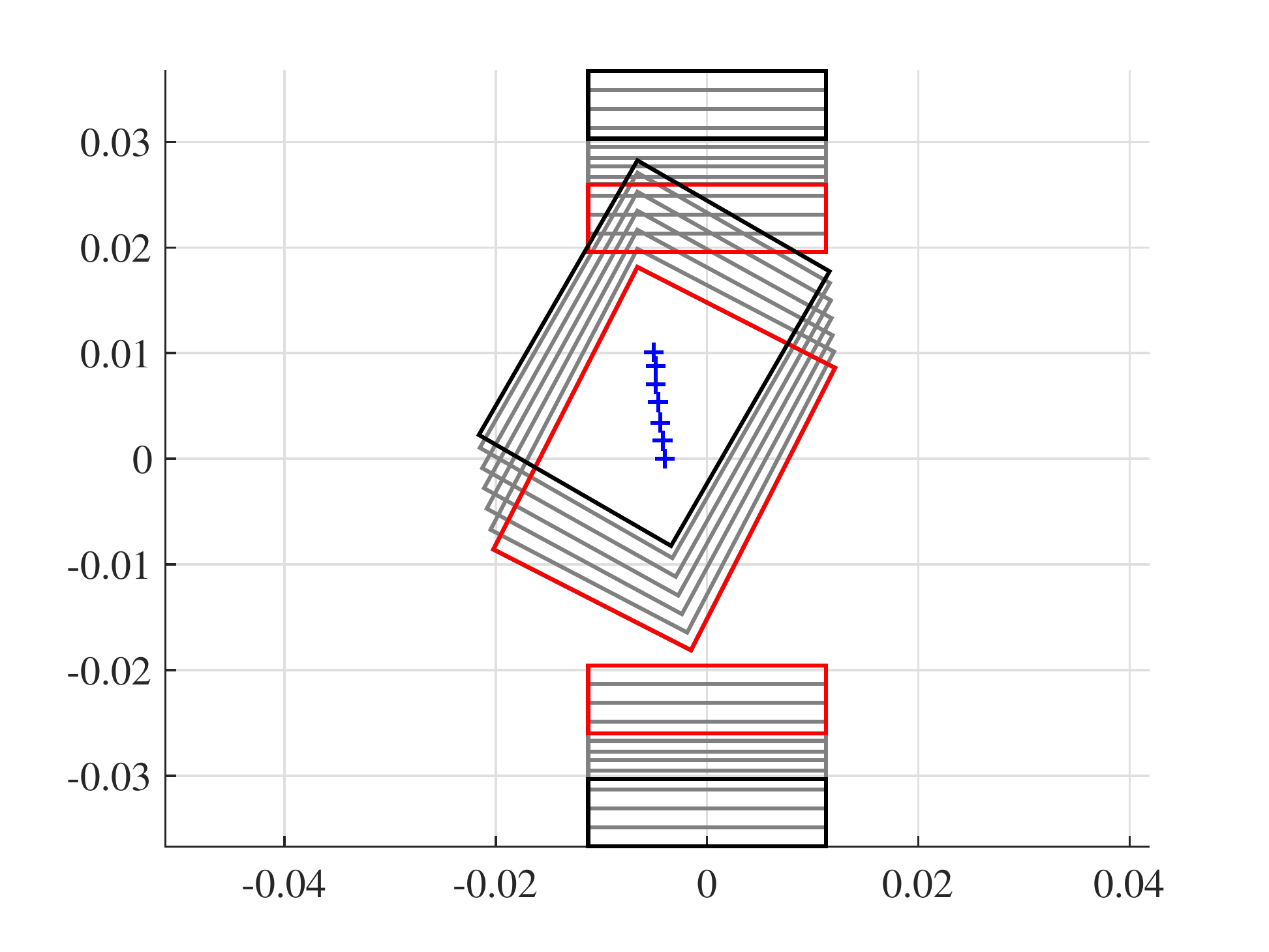}
\caption{Jamming.\\}
\label{fig:wedge_grasp}
\end{subfigure}
~
\begin{subfigure}[t]{\figbarlen}
\centering
\includegraphics[width=\figbarlen]{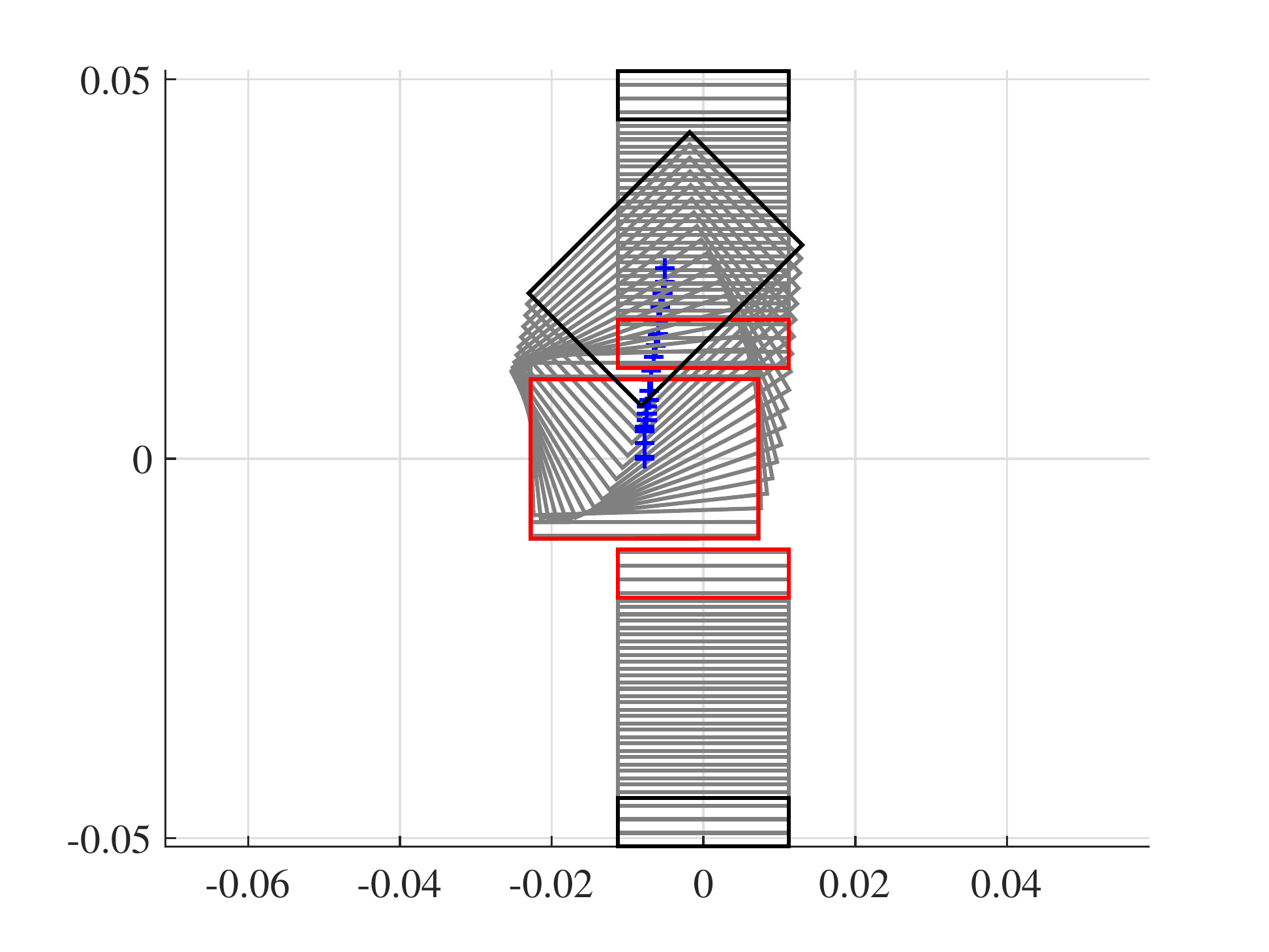} 
\caption{Grasped with offset.\\}
\label{fig:offset_grasp2}
\end{subfigure}
~
\begin{subfigure}[t]{\figbarlen}
\centering
\includegraphics[width=\figbarlen]{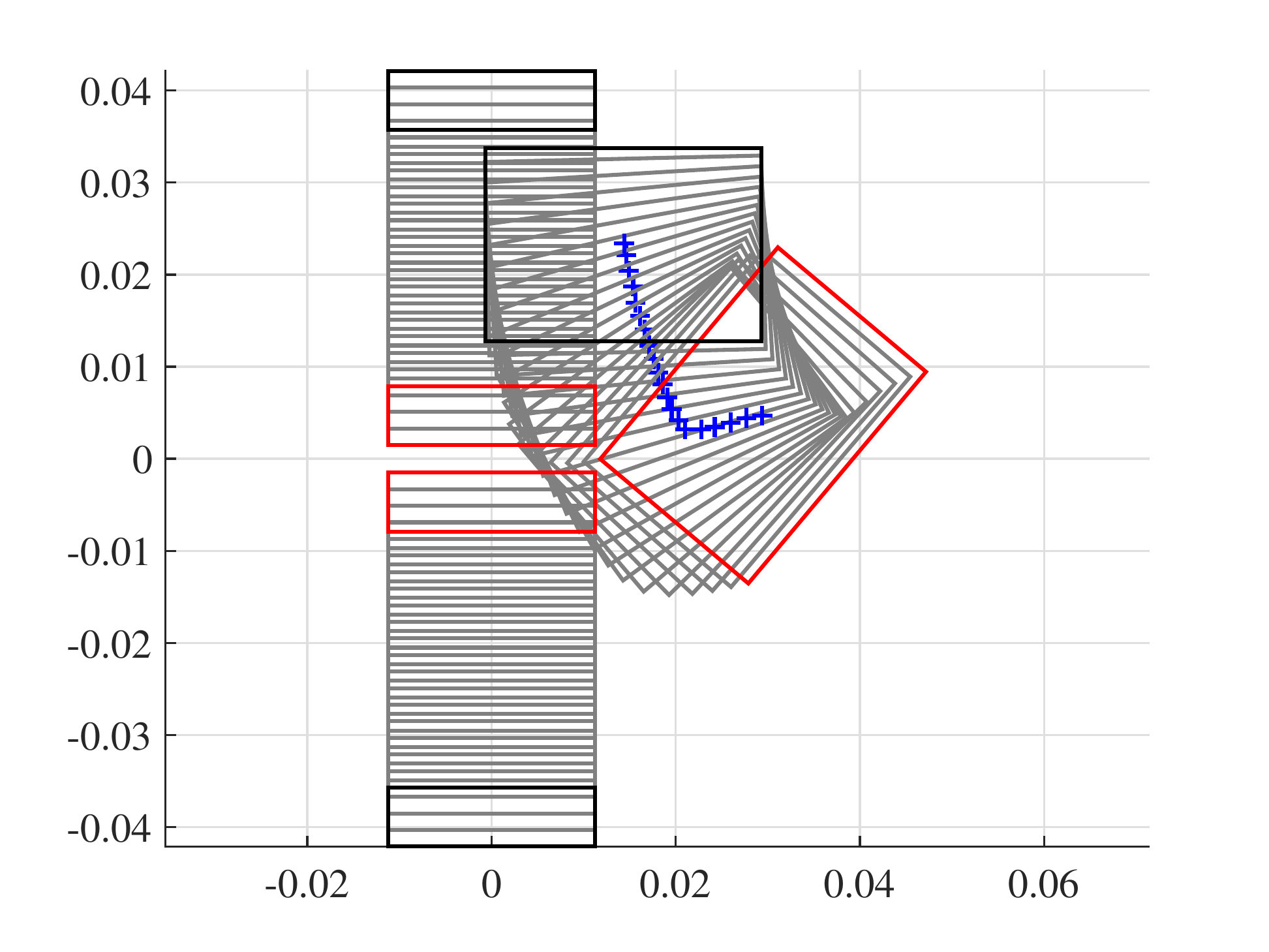}
\caption{Slipped to free space.}
\label{fig:slip_grasp}
\end{subfigure}
\caption{Simulation results using the proposed
  contact model illustrating the process of a parallel jaw gripper
  squeezing along the y axis when the object is placed at different
  initial poses. The initial, final and intermediate gripper
  configurations and object poses are
  in black, red and grey respectively. Blue plus signs trace out the
  center of mass trajectory of the object. }
\label{fig:grasp_failure}
\vspace{-0.1in}
\end{figure}

The proposed contact model is a direct extension of \cite{Zhou16}, 
which presents a dual mapping between an applied wrench and a resultant
object twist. In this paper, we map a given position controlled input (which
is common in most standard industrial manipulators) to the resultant
twist including the no-motion case for jamming and grasping. The applied
wrench is solved as an intermediate output without needing to control it. 
The rest of the paper is organized as follows:
\begin{itemize}
\item Section \ref{sec:related_work} describes the previous work.
\item Section \ref{sec:background} reviews the convex polynomial
  representation of the limit surface \cite{Zhou16} and the mechanics of pushing.
\item Section \ref{sec:single_contact} develops the contact model of
  unilateral frictional contact for both slipping and
  sticking. Section \ref{sec:multiple_contacts} develops the model for
  multiple frictional contacts.
\item Section \ref{sec:stochasic} develops sampling strategies of physically consistent model
  parameters that captures the inherent frictional stochasticity.
\item Section \ref{sec:expvalidations} demonstrates experimental
  evaluation of both pushing and grasping simulation using the
  proposed contact model.
\end{itemize} 
We assume quasi-static rigid body planar mechanics \cite{Mason1986} where
inertia forces and out-of-plane moments are negligible.

\section{Related Work} \label{sec:related_work}
The mechanics of pushing and grasping
involving finite object motion with frictional support was first studied in \cite{Mason1986a}.
A notable result is the voting theorem which dictates the sense of
rotation given a push action and the center of pressure regardless of
the uncertain pressure distribution. Brost \cite{brost1988automatic}
used this result to construct the operational space for planning squeezing
and push-grasping actions under uncertainty. However, many unrealistic
assumptions are made in order to reduce the state space and create
finite discrete transitions, including infinitely long fingers
approaching the object from infinitely far away. Additionally, how far
to push the object in the push-grasp action is not addressed.    
Peshkin and Sanderson \cite{Peshkin1988a} provided an analysis on the slowest speed of
rotation given a single point push. They \cite{Peshkin1988} used this result to design
fences for parts feeding. 
Lynch and Mason \cite{Lynch1996e} derived conditions for stable edge
pushing such that the object will remain attached to the pusher
without slipping or breaking contact. 
All of these results do not assume knowledge of the pressure
distribution except the location of center of pressure. They can be
classified as worse case guarantees without looking into the details of sliding
motion. Despite being agnostic to pressure distribution, these methods tend to be
overly conservative. 

Another line of research is to identify the necessary physical
parameters. Common approaches \cite{Lynch1993, Yoshikawa1991} discretize the contact patch into
grids and optimize for approximate criteria of force
balancing. These methods naturally suffer from the downside of
coarse discrete approximation of distributions or curse of
dimensionality if fine discretization is adopted. 
Additionally, the instantaneous center of rotation of the object can
coincide with one of the support points, rendering the kinematic
solution computationally hard due to combinatorial sliding/sticking mode assignment for each support point.   

 Goyal et al. \cite{Goyal1991} noted that all the possible static and sliding frictional
wrenches, regardless of the pressure distribution, form a convex set whose
boundary is called as limit surface. 
The limit surface can be constructed from Minkowsky sum
of frictional limit curves at individual support points without a convenient explicit form. Howe and Cutkosky \cite{howe1996practical} presented an experimental
method to identify an ellipsoidal approximation given known pressure. 
Dogar and Srinivasa \cite{Dogar2010pgd} used the ellipsoidal
approximation and integrated with motion planners to plan push-grasp
actions for dexterous hands.
Closely related to our work, Lynch et al. \cite{lynch1992manipulation} derived the
kinematics of single point pushing with centered and axis aligned
ellipsoid approximation. 
Zhou et al. \cite{Zhou16} proposed a framework of representing
planar sliding force-motion models using homogeneous even-degree
sos-convex polynomials, which can be identified by solving a semi-definite programming. The set of
applied friction wrenches is the 1-sublevel set of a convex polynomial whose
gradient directions correspond to incurred sliding body twist.
In this paper, we extend the convex polynomial model to associate a
commanded rigid position-controlled end effector motion to the instantaneous resultant
object motion. 
We show that single contact with convex quadratic limit
surface model has a unique analytical linear solution which extends
\cite{lynch1992manipulation}. The case for a high order convex polynomial model is reduced to solving a sequence of such subproblems. 
For multiple contacts (e.g., pushing with multiple points or grasping) we need to add linear
complementarity constraints \cite{stewart1996implicit} at the pusher
points, and the entire problem is a standard linear complementarity problem (LCP).

\section{Notations and Background}
 \label{sec:background}
 We first introduce the following notations:
\begin{itemize}
\item $O$: the object center of mass used as the origin of the
  body frame. We assume vector quantities are with respect to body
  frame unless specially noted.
\item $R$: the region between the object and the supporting
  surface.
\item $\mathbf{f_s}(\mathbf{r})$: the friction force distribution
  function that maps a point $\mathbf{r}$ in the contact area $R$ to its friction
  force the object applies on the supporting surface. For isotropic point Coulomb friction law, when the velocity
  at $\mathbf{r}$ is nonzero,  $\mathbf{f_s}(\mathbf{r})$ is in the
    same direction of the velocity. Its magnitude equals the
    pressure force multiplied by the coefficient of friction between the object and the supporting surface. When the object is static, $\mathbf{f_s}(\mathbf{r})$ is
  indeterminate and dependent on the externally applied force by the manipulator.
\item $\mathbf{V} = [V_x; V_y; \omega]$: the body twist (generalized
  velocity). 
\item $\mathbf{F} = [F_x; F_y; \tau]$: the applied body wrench by the manipulator that quasi-statically balances the friction wrench from the surface.
\item $\mathbf{p}_i$: each contact point between the manipulator end effector and
  object in the body frame.
\item $\mathbf{v}_{p_i}$: applied velocities by the manipulator end effector at
  each contact point in the body frame.  
\item $\mathbf{n}_{p_i}$: the inward normal at contact point
  $\mathbf{p}_i$ on the object.
\item $\mu_c$: coefficient of friction between the object and the manipulator
  end effector.

\end{itemize}
\subsection{Force-Motion Model}
In this section we review the basics of force-motion models for planar
sliding and the mechanics of pushing. We refer the readers to
\cite{Goyal1989, Mason1986a, Zhou16} for more details. 
Given a body twist $\mathbf{V}$, the  components of the friction
wrench $\mathbf{F}$ are given by integrations over $R$:  
\begin{align}
[F_x;F_y]  = \int_{R} \mathbf{f}_{s}(\mathbf{r})\,dr, \quad \tau =
\int_{R} \mathbf{r} \times \mathbf{f}_{s}(\mathbf{r}) \,dr.
\end{align}  
We can compute $\mathbf{F}$ for each $\mathbf{V}$ and form
the set of all possible friction wrenches. Goyal et
al. \cite{Goyal1989} defined the set boundary
as limit surface. It is shown that the friction wrench set is
convex and points on the limit surface correspond to friction
wrenches when the object slides. Additionally, the normal for a point (wrench) on the limit surface is parallel to the corresponding twist.
Zhou et al \cite{Zhou16} showed that level sets of homogeneous even
degree convex polynomials can approximate the limit
surface geometry sufficiently well. Denote by $H(\mathbf{F})$ the convex polynomial function, the twist $\mathbf{V}$ for a given friction wrench $\mathbf{F}$ is parallel to the gradient $\nabla H(\mathbf{F})$:
\begin{align}
\mathbf{V} &= k \nabla H(\mathbf{F}) \qquad k>0.
\end{align}
Additionally the inverse mapping can be efficiently computed. Given the twist $\mathbf{V}$, optimizing a least-squares objective
with the Gauss-Newton algorithm gives the unique
solution that corresponds to the wrench $\mathbf{F}$.

With a position-controlled manipulator, we are given contact points $\mathbf{p}$ with inward normals $\mathbf{n}_p$,
pushing velocities $\mathbf{v_p}$ and coefficient of friction $\mu_c$
between the pusher and the object. The task is to resolve the incurred
body twist $\mathbf{V}$ and consistent contact modes (sticking, slipping, breaking contact) to maintain wrench balance.

\begin{figure}[ht!]
\centering
\begin{subfigure}[t]{\mylenlong}
\centering
\includegraphics[width=\mylenlong]{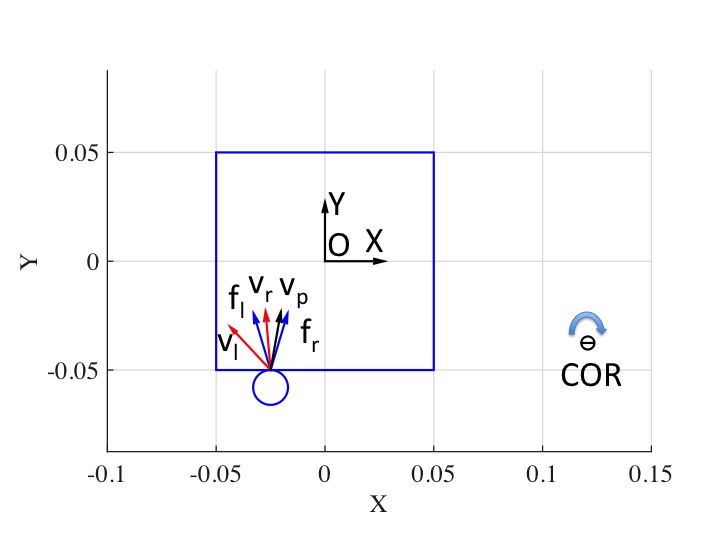} 
\caption{ The square has a uniform pressure
  distribution over 100 support grid points sharing the same
  coefficient of friction. 
The finger's pushing velocity is to the
  right of the motion cone and hence the finger will slide to the right. The instantaneous clockwise center of rotation is marked as a circle with a negative sign.}
\label{fig:motion_cone}
\end{subfigure}
~
\begin{subfigure}[t]{\mylenlong}
\centering
\includegraphics[width=\mylenlong]{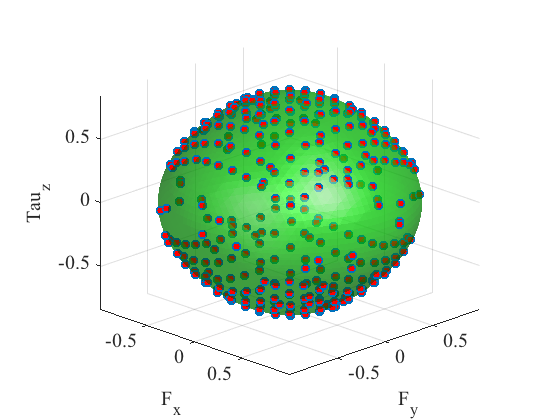}
\caption{Corresponding convex fourth order polynomial level set representation identified \cite{Zhou16} from two hundred random wrench twist pairs.\\} 
\vspace{-0.15in}
\label{fig:limit surface}
\end{subfigure} 

\caption{Mechanics of single point pushing and a fourth order
  representation of the limit surface.}
\label{fig:robot_exp_compare}
\vspace{-0.15in}
\end{figure}

\section{Contact Modelling}
\subsection{Single point pusher}
 \label{sec:single_contact}
Let the COM be the point of origin of the local body frame a level set
representation of limit surface $H(\mathbf{F})$.  We introduce the concept of motion cone first
proposed in \cite{Mason1986a}. Let  $J_p =
\begin{bmatrix} 
1 & 0 & -p_y\\
0 & 1 & p_x
\end{bmatrix}$,
and denote by $\mathbf{F}_l = J_{p}^T\mathbf{f}_l $ and $\mathbf{F}_r
= J_{p}^T\mathbf{f}_r $ the left and right edges of the applied
wrench cone with corresponding resultant twist directions $\mathbf{V}_l = \nabla
  H(\mathbf{F}_l)$ and $\mathbf{V}_r = \nabla
  H(\mathbf{F}_r)$ respectively. The left edge of the
motion cone is $\mathbf{v}_l = J_{p} \mathbf{V}_l$ and the right edge of the motion
cone is $\mathbf{v}_r = J_{p} \mathbf{V}_r$. 
If the contact point pushing velocity $\mathbf{v_p}$ is inside 
the motion cone, i.e., $\mathbf{v}_p \in \mathbf{K}(\mathbf{v}_l ,
\mathbf{v}_r) $, the contact sticks. 
When $\mathbf{v}_p$ is outside the motion cone, sliding occurs. If $\mathbf{v}_p$ is to the left of
$\mathbf{v}_l$, then the pusher slides left with respect to the
object. Otherwise if $\mathbf{v}_p$ is to the right of $\mathbf{v}_r$,
then the pusher slides right as shown in
Fig. \ref{fig:motion_cone}.

The following constraints hold assuming sticking contact:
\begin{align}
v_{px} &\quad=\quad  V_x - \omega p_y  \label{ctr1}\\
v_{py} &\quad = \quad  V_y + \omega p_x \label{ctr2}\\
\mathbf{V} & \quad  = \quad  k\cdot\nabla H(\mathbf{F}), \qquad k>0 \label{ctr3} \\
\tau & \quad= \quad  -p_y F_x + p_xF_y \label{ctr4}
\end{align}
In the case of ellipsoidal (convex quadratic) representation, i.e., $H(\mathbf{F}) =
\mathbf{F}^TA \mathbf{F}$ where $A$ is a positive definite matrix, the problem
is a full rank linear system with a unique solution. Lynch et al. \cite{lynch1992manipulation} gives
an analytical solution when $A$ is diagonal. We show that a unique analytical solution exists for any
positive definite symmetric matrix $A$. Let $\mathbf{t} = [-p_y, p_x, -1]^T $, $D =  [J_p^T,
A^{-1}\mathbf{t}]^T$ and $\mathbf{V}_{p} = [v_{p}^{T}, 0]^T$,
equations \ref{ctr1}-\ref{ctr4} can then be combined as one linear equation:
\begin{align}
\mathbf{V} = D^{-1}\mathbf{V}_{p} \label{eq:sticking_ellipsoid}
\end{align}

\begin{theorem}
\textit{Pushing with single sticking contact and the convex quadratic representation of limit surface (abbreviated as $P.1$) has a unique solution from a linear system.}  
\end{theorem}
\begin{proof}
It is obvious that we only need to prove $D$ is invertible.
1) The row vectors of $J_p$ are linearly independent and span a plane.
2) $J_p\mathbf{t} = 0$ implies $\mathbf{t}$ is
orthogonal to the spanned plane. 
3) If $D$ is not full rank, then
$A^{-1}\mathbf{t}$ must lie in the spanned plane and is therefore orthogonal to $\mathbf{t}$. This contradicts with the fact that $\langle \mathbf{t}, A^{-1}\mathbf{t}
\rangle > 0$ for positive definite matrix $A^{-1}$ and nonzero vector
$\mathbf{t}$. 
\end{proof}
\begin{corollary}
Pushing with single sticking contact and the general convex polynomial limit surface representation is reducible to solving a sequence of sub-problems $P.1$. 
\end{corollary}

For general convex polynomial representation $H(\mathbf{F})$, the
following optimization is equivalent to equation \ref{ctr1}-\ref{ctr4}: 
\begin{align}
  & \underset{\mathbf{F}}{\text{minimize}}
  & & \| J_p\nabla H(\mathbf{F}) - \mathbf{v_p}\| \label{eq:opt_goal}\\
 & \text{subject to}
 && \mathbf{t}^TF = 0
\end{align}
When $H(\mathbf{F})$ is of convex quadratic (ellipsoidal) form, the analytical minimizer
is $\mathbf{F} = A^{-1}D^{-1}\mathbf{V}_{p}$.  In the case of high
order convex homogeneous polynomials, we can resort to an iterative
 solution where we use the Hessian matrix as a local ellipsoidal
 approximation, i.e., set $A_t = \nabla^2 H(\mathbf{F_t}) $ and compute $\mathbf{F_{t+1}} =
 A_t^{-1}D^{-1}\mathbf{V}_{p}$ until convergence.  

When $\mathbf{v}_p$ is outside of the motion cone, assuming right sliding occurs without loss of generality, the wrench applied by the finger equals $\mathbf{F}_r$. The resultant object twist $\mathbf{V}$
follows the same direction as $\mathbf{V}_r$ with proper magnitude such that the contact is maintained:
\begin{align}
\mathbf{V} = s \mathbf{V}_r \\
s = \frac{\mathbf{n_p}^T \mathbf{v}_p}{\mathbf{n_p}^T \mathbf{v}_l} 
\end{align}

\subsection{Multi-contacts}
  \label{sec:multiple_contacts}
Mode enumeration is tedious for multiple contacts. The linear complementarity formulation for frictional contacts~(\cite{stewart1996implicit}) provides a convenient representation.  
Denote by $m$ the total number of contacts, the quasi-static force-motion equation is given by:
\begin{align}
\mathbf{V} &= k\nabla H(\mathbf{F}),\label{eq:lcpstart}
\end{align}
where the total applied wrench is the sum of normal and frictional
wrenches over all applied contacts:
\begin{align}
\mathbf{F} &= \sum_{i=1}^m J_{p_i}^T (f_{n_i} \mathbf{n}_{\mathbf{p_i}} +
D_{\mathbf{p_i}} \mathbf{f}_{t_i}).
\end{align}
$f_{n_i}$ is the normal force magnitude along the normal
$\mathbf{n_i}$, and $\mathbf{f}_{t_i}$ is the vector of tangential friction force magnitudes along the column vector basis of $D_{\mathbf{p_i}} =[\mathbf{t}_{p_i}, -\mathbf{t}_{p_i}]$.
The velocity at contact point $\mathbf{p_i}$ on the object is given by $J_{p_i}
\mathbf{V}$. 
The first order complementarity constraints on the normal force magnitude 
and the relative velocity are given by:  
\begin{align}
0 \leq f_{n_i} &\perp (\mathbf{n}_{p_{i}}^T( J_{p_i} \mathbf{V} -
\mathbf{v}_p)) \geq 0. 
\end{align}
The complementarity constraints for Coulomb friction are given by: 
\begin{align}
0 \leq \mathbf{f}_{t_i} & \perp ({D}_{\mathbf{p_i}}^T( J_{p_i} \mathbf{V} -
\mathbf{v}_p) + \mathbf{e} \lambda_i) \geq 0, \\
0 \leq \lambda_i & \perp (\mu_if_{n_i} - \mathbf{e}^T\mathbf{f}_{t_i}) \geq 0, \label{eq:lcpend}
\end{align}
where $\mu_i$ is the coefficient of friction at $\mathbf{p_i}$ and $\mathbf{e} = [1;1]$. 
In the case of ellipsoid (convex quadratic) representation, i.e., $H(\mathbf{F}) =
\mathbf{F}^TA \mathbf{F}$ where $A$ is a positive definite matrix, equations
 \ref{eq:lcpstart} to \ref{eq:lcpend} can be written in matrix form: 
\begin{align}
\label{eq:lcpmatrix}
\begin{bmatrix}
0 \\
\alpha \\
\beta \\
\gamma 
\end{bmatrix} & = 
\begin{bmatrix}
A^{-1}/k & -N^T & -L^T & 0 \\
N & 0 & 0 & 0 \\
L & 0 & 0 & E \\
0 & \mathbf{\mu} & -E^T & 0
\end{bmatrix}
\begin{bmatrix}
\mathbf{V} \\
\mathbf{f}_n \\
\mathbf{f}_t \\
\lambda
\end{bmatrix}
+ \begin{bmatrix}
0 \\
\mathbf{a} \\
\mathbf{b} \\
0
\end{bmatrix}, \\ \nonumber
 0 & \leq \begin{bmatrix}
\alpha \\
\beta \\
\gamma 
\end{bmatrix} \perp
\begin{bmatrix}
\mathbf{f}_n \\
\mathbf{f}_t \\
\lambda
\end{bmatrix} \geq 0, 
\end{align} 
where the binary matrix $E \in R^{2m \times m}$ equals $ 
\begin{bmatrix} 
\mathbf{e} & & \\
& \ddots & \\
& & \mathbf{e} 
\end{bmatrix}$,  $\mathbf{\mu} = [\mu_1,\dots, \mu_m ]^T$, the stacking
matrix $N \in R^{m\times 3}$ equals $[\mathbf{n}_{p_{1}}^TJ_{p_1}; \dots;\mathbf{n}_{p_{m}}^TJ_{p_m} ]$,  the stacking
matrix $L \in R^{2m\times 3}$ equals $[D_{p_1}^T J_{p_1};\dots;D_{p_m}^T J_{p_m} ]$, 
the stacking vector $\mathbf{s_a} \in R^{m}$ equals $[-\mathbf{n}_{p_1}^T\mathbf{v}_{p_1} ,\dots,
-\mathbf{n}_{p_m}^T\mathbf{v}_{p_m}]^T$ and vector $\mathbf{s_b} \in
R^{2m}$ equals $[-D_{p_1}^T\mathbf{v}_{p_1} ,\dots,
-D_{p_m}^T\mathbf{v}_{p_m} ]^T$.
 
Note that the positive scalar $k$ will not affect the solution value
of $\mathbf{V}$ since $\mathbf{f}_n$ and $\mathbf{f}_t$ will scale
accordingly. Hence, we can drop the scalar $k$ and further substitute
$\mathbf{V} = A(N^T\mathbf{f}_n + L^T\mathbf{f}_t)$ into
equation \ref{eq:lcpmatrix} and reach the standard linear
complementarity form as follows: 
\begin{align}
\label{eq:lcp_std} 
\begin{bmatrix}
\alpha \\
\beta \\
\gamma 
\end{bmatrix} & = 
\begin{bmatrix}
NAN^T &NAL^T & 0 \\
LAN^T & LAL^T & E \\
\mathbf{\mu} & -E^T & 0
\end{bmatrix}
\begin{bmatrix}
\mathbf{f}_n \\
\mathbf{f}_t \\
\lambda
\end{bmatrix}
+ \begin{bmatrix}
\mathbf{s_a} \\
\mathbf{s_b} \\
0
\end{bmatrix}, \\ \nonumber
 0 & \leq \begin{bmatrix}
\alpha \\
\beta \\
\gamma 
\end{bmatrix} \perp
\begin{bmatrix}
\mathbf{f}_n \\
\mathbf{f}_t \\
\lambda
\end{bmatrix} \geq 0.
\end{align}
Similarly, for the case of high order convex homogeneous polynomials, we can iterate between taking the linear Hessian approximation and solving the LCP problem in equation~\ref{eq:lcp_std} until convergence. 
\begin{lemma}
The object is quasi-statically jammed or grasped if the LCP problem (equation~\ref{eq:lcp_std}) yields no solution.
\end{lemma}
Fig. \ref{fig:jamming} provides a graphical proof. When equation \ref{eq:lcp_std} yields no solution, either there is no feasible
kinematic motion of the object without penetration or all the
friction loads associated with the feasible instantaneous twists cannot balance against any element from the set of possible applied wrenches. In this case, the object is quasi-statically jammed or grasped between the fingers. Neither the object nor the end effector can move.  
\begin{figure}[ht!]
\centering
\includegraphics[width=\mylenlong]{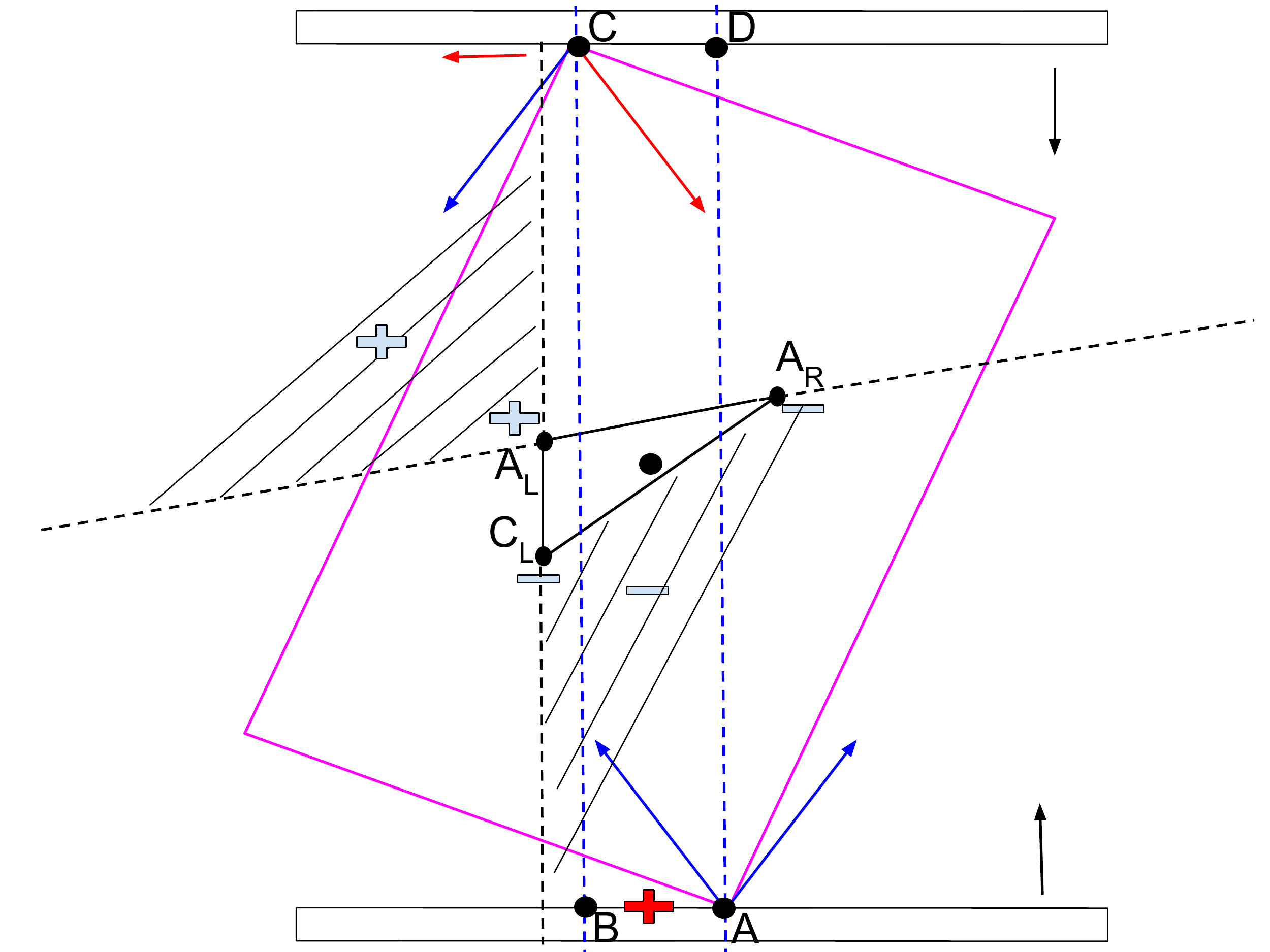} 
\caption{Using moment labeling (\cite{mason2001mechanics}), the center of rotation (COR) has positive sign (counter-clockwise) and can only lie in the band
  between the two blue contact normal lines. Further, the COR must lie on segment AB (contact point A sticks) or segment CD (contact point C sticks) since otherwise both contacts will slip, but the total wrench from the
  two left edges of the friction cones has negative moment which cannot
  cause counter-clockwise rotation. Without loss of generality, we can
  assume COR (red plus) lies on segment AB, leading to
  sticking contact at A and left sliding at C. Following a similar
  analysis using the force dual
  graphical approach (\cite{brost1991graphical}), each single friction force can be mapped to its instantaneous resultant signed COR whose convex
  combination forms the set of all possible CORs under the composite
  friction forces.  The COR can either be
  of positive sign in the upper left hatched region or negative sign in the
  lower right hatched region which contradicts with the proposed AB segment.  
Hence jamming occurs and neither the gripper nor the
  object can move. This corresponds exactly to the no solution case of
  equation \ref{eq:lcp_std}.\\ }
\vspace{-0.15in}
\label{fig:jamming}
\end{figure}

\section{Stochasticity} \label{sec:stochasic}
Uncertainty is inherent in frictional interaction. 
Two major sources contribute to the uncertainty in planar motion: 1)
indeterminancy of the supporting friction distribution $f_{s}(r)$ due
to changing pressure distribution and coefficients of friction between the object and
support surface; 2) the coefficient of friction $\mu_c$ between the object and the robot end effector.
We sample $\mu_c$ uniformly from a given range. To model the effect of changing support friction
distribution, for a even degree-$d$ strictly convex polynomial (except at the point of origin) $H(\mathbf{F};a) =\sum_{i=1}^{m} a_iF_x^{i_1}F_y^{i_2}F_z^{d-i_1-i_2}$ with $m$
monomial terms~\cite{Zhou16}, we sample the polynomial coefficient parameters $a$ in from a distribution that satisfies:
\begin{enumerate}
\item Samples from the distribution should result in a even degree
  homogeneous convex polynomial to represent the limit surface. 
\item The mean can be set as a prior estimate and the amount of variance controlled by one parameter.
\end{enumerate}
The Wishart distribution $S \sim W(\hat{S}, n_{df})$
\cite{wishart1928generalised} with mean $n_{df}S_{est}$ and variance
$Var(S_{ij}) = n_{df}(\hat{S}_{ij}^2 + \hat{S}_{ii}\hat{S}_{jj})$ is
defined over symmetric positive semidefinite random matrices as a
generalization of the chi-squared distribution to multi-dimensions.
For ellipsoidal (convex quadratic) $H(\mathbf{F}; A) = A^T\mathbf{F}A$, we can directly
sample $A$ from $\frac{1}{n_{df}}W(A_{est}, n)$ with mean $A_{est}$ and variance $Var(A_{ij}) = \frac{(\hat{A_{est}}_{ij}^2 + \hat{A_{est}}_{ii}\hat{A_{est}}_{jj})}{n_{df}}$ where $A_{est}$ is some
estimated value from data or fitted for a particular pressure
distribution. Sampling from general convex polynomials is
hard. Fortunately, we find that sampling from the 
sos-convex \cite{parrilo2000structured,magnani2005tractable}
polynomials subset is not. The key is the coefficient vector $a$ of a sos-convex
polynomial $H(\mathbf{F};a)$ has a unique one-to-one mapping to a positive definite
matrix $Q$ so that we can first sample $\tilde{Q}$ from $\frac{1}{n_{df}} W(Q_{est},
n_{df})$ \footnote{We have noted that adding a small constant on the
  diagonal elements of $\tilde{Q}$ improves numerical stability.}
and then map back to $\tilde{a}$. Given a sos-convex polynomial
representation of $H(\mathbf{F};a)$, the Hessian matrix $\nabla^2
H(\mathbf{F};a)$ at $\mathbf{F}$ is positive definite, i.e., for any non-zero vector  $\mathbf{z} \in \mathbb{R}^3$,
there exists a positive-definite matrix $Q$ such that   
\begin{align}
\mathbf{z}^T\nabla^2 H(\mathbf{F};a) \mathbf{z}  &= y(\mathbf{F},\mathbf{z})^T Q y(\mathbf{F},\mathbf{z}) > 0. \label{eq:sos-matrix}
\end{align}
In the case of fourth order polynomial we have $y(\mathbf{F},\mathbf{z}) = [z_1F_x,z_1F_y, z_1F_z, z_2F_x, z_2F_y, z_2F_z, z_3F_x, z_3F_y, z_3F_z]^T$.
$Q$ and $a$ are related through a set of
linear equalities: equation (\ref{eq:sos-matrix}) can be
written as a set of $K$ sparse linear constraints on $Q$ and $a$.
\begin{align} 
\Tr(C_{k} Q) &= b_k^Ta , \quad k\in \{1 \dots K\} 
\end{align}
where $C_k$ and $b_k$ are the constant sparse element indicator matrix and
vector that only depend on the polynomial degree $d$.
Hence we can map each sampled $\tilde{Q}$ back to $\tilde{a}$. The
degree of freedom parameter $n_{df}$ determines the sampling
variance. The smaller $n_{df}$ is, the noisier the system will be.

\section{Experimental Evaluation} \label{sec:expvalidations}
\subsection{Evaluation of Deterministic Pushing Model}
We evaluate our deterministic model on the large scale MIT pushing
dataset \cite{yu2016more} and a smaller dataset \cite{Zhou16} that has discrete
pressure distributions. For the MIT pushing dataset, we use 10mm/s
velocity data logs for 10 objects\footnote{Despite having the same experimental set up and similar
  geometry and friction property to the other two triangular shapes, the results for object Tr2 is
  about 1.5 -2 times worse. Due to time constraint, we have not ruled
  out the possibility that the data for object Tr2 is corrupted.} on 3 hard surfaces including delrin, abs and plywood. 
The force torque signal is first filtered with a low pass filter and 5
wrench-twist pairs evenly spaced in time are extracted from each push
action json log file. 10 random train-test splits (20 percent of the logs
for training, 10 percent for validation and the rest for
testing) are conducted for each object-surface scenario. On average,
around 600 wrench-twist pairs are used for identification. 

Given two poses $q_1 = [x_1, y_1, \theta_1]$ and $q_2 = [x_2, y_2,
\theta_2]$, we define the deviation metric $d(q_1, q_2)$ which
combines both the displacement and angular offset as $d(q_1, q_2) =
\sqrt{(x_1 - x_2)^2 + (y_1 - y_2)^2} + \rho\cdot\min( |\theta_1 -
\theta_2|, 2\pi - |\theta_1 -
\theta_2| )$, where $\rho$ is the characteristic length of the object (e.g., radius of gyration or radius of minimum circumscribed circle).
A one dimensional coarse grid search
over the coefficient of friction $\mu_c$ between the pusher and object
is chosen to minimize average deviation of the predicted final pose
and ground truth final pose on training data. Table
\ref{table:result_mcube} shows the average metric with a 95\% percent
confidence interval. Interestingly, we find that using more training data does not improve the performance much. This is likely due to the inherent
  stochasticity (variance) and changing surface conditions as reported in \cite{yu2016more}.

The objects in the MIT pushing dataset are closer to uniform
pressure. We also evaluate on a smaller dataset \cite{Zhou16} that has discrete
pressure distributions, and in particular three points support whose
pressure can be derived exactly as ground truth. We use 400 wrench twists
sampled from the ideal limit surface for training. The coefficient of friction
between the object and pusher is determined by a grid search over
40 percent of the logs to determine. We use the remaining 60 percent to
evaluate simulation accuracy. Note that in both evaluations, the accuracy of deterministic models are upper bounded by the system variance.    

\begin{table*}[]  
\footnotesize                                                                                                                                              
\centering 
\setlength\tabcolsep{4pt} 
\begin{tabular}{|l|c|c|c|c|c|c|c|c|c|c|c|}   
\hline
&rect1 & rect2 & rect3 & tri1 & tri3 & ellip1 & ellip2 & ellip3
&hex & butter \\                                                                                          
\hline                                                                                                                                  
poly4-delrin & 8.28$\pm$0.29 & 5.37$\pm$0.23 & 6.10$\pm$0.21 & 9.71$\pm$0.33& 7.54$\pm$0.23 & 7.68$\pm$0.51 & 8.90$\pm$1.40 & 7.35$\pm$0.38 & 6.38$\pm$0.28 & 4.83$\pm$0.27 \\
\hline                                                                                                                                                       
quad-delrin &8.60$\pm$0.35 & 5.92$\pm$0.14 & 8.20$\pm$0.16 & 9.90$\pm$0.41 & 8.18$\pm$0.15 & 6.85$\pm$0.25 & 6.29$\pm$0.24 & 8.08$\pm$0.51 & 6.42$\pm$0.12 & 5.97$\pm$0.23 \\
\hline                                                                                  
delrin & 35.48 & 40.53 & 35.98 & 36.91 & 34.66 & 32.18 & 38.05 & 33.37 & 33.55 & 34.09 \\
\hline                                                                                                                                                       
poly4-abs &5.86$\pm$0.11 & 7.48$\pm$0.80 & 3.59$\pm$0.12 & 7.13$\pm$0.26  & 5.17$\pm$0.38 & 8.45$\pm$1.13 & 9.18$\pm$1.26 & 5.93$\pm$0.19 & 7.56$\pm$0.39 & 3.94$\pm$0.11 \\ 
\hline                                                                                                                                                       
quad-abs&6.07$\pm$0.16 & 6.74$\pm$0.27 & 6.19$\pm$0.18 & 8.00$\pm$0.37 & 7.17$\pm$0.37 & 6.66$\pm$0.28 & 7.69$\pm$0.27 & 5.78$\pm$0.21 & 8.19$\pm$0.21 & 5.39$\pm$0.15 \\ 
\hline                                                                                  
abs&34.14 & 39.74 & 33.98 & 35.43 & 32.37 & 32.68 & 33.53 & 32.45 & 33.23 & 33.53 \\
\hline                                                                                                                                                       
poly4-plywood&6.86$\pm$0.71 & 6.86$\pm$0.13 & 5.93$\pm$0.33 & 4.61$\pm$0.13 & 7.21$\pm$0.47 & 4.39$\pm$0.16 & 4.99$\pm$0.31 & 5.72$\pm$0.31 & 8.41$\pm$0.24 & 4.72$\pm$0.17 \\
\hline                                                                                                                                                       
quad-plywood&6.20$\pm$0.20 & 7.22$\pm$0.18 & 6.88$\pm$0.18 & 5.96$\pm$0.19 & 9.43$\pm$0.56 & 4.42$\pm$0.12 & 5.84$\pm$0.20 & 6.46$\pm$0.26 & 8.85 $\pm$0.17 & 6.05$\pm$0.22 \\
\hline                                                                                  
plywood&31.86 & 33.22 & 32.94 & 32.81 & 33.78 & 27.24 & 28.23 & 33.29 & 32.77 & 34.10 \\
\hline                                                                                                                                                       
\end{tabular}                                                                                                                                                
\caption{Average deviation (in mm) that combines both displacement and angular offset) between the simulated final pose and actual final pose with 95 percent confidence interval. The 3rd, 6th and 9th
rows are the deviation from the ground truth initial pose and final pose to indicate how much the object is moved due to the push. In most cases, the fourth order convex (poly4) polynomial has better accuracy. The average normalized percentage error for poly4 is 20.05\% and for quadratic is 21.39\%. However, the accuracy of a fixed deterministic model is bounded by the inherent variance of the system.}                                                                                                                                     
\label{table:result_mcube}        
\end{table*}

\begin{table}[htbp!]                                           
\centering                                              
\begin{tabular}{|l|c|c|c|c|}       
\hline
&3pts1 & 3pts2 & 3pts3 & 3pts4 \\                          
\hline                                                  
poly4-hardboard & 3.52$\pm$0.21 & 2.75$\pm$0.25 & 2.92$\pm$0.27 & 2.80$\pm$0.23 \\
\hline                                                  
quad-hardboard &3.82$\pm$0.24 & 3.63$\pm$0.27 & 3.35$\pm$0.23 & 3.96$\pm$0.28 \\
\hline
hardboard & 16.63 & 13.86 & 14.83 & 15.15 \\
\hline                                                 
poly4-plywood & 3.78$\pm$0.11 & 2.80$\pm$0.15 & 2.84$\pm$0.16 & 3.26$\pm$0.11 \\
\hline                                                  
quad-plywood & 4.24$\pm$0.15 & 3.56$\pm$0.17 & 3.28$\pm$0.08 & 4.12$\pm$0.13 \\
\hline          
plywood & 16.56 & 13.81 & 15.27 & 14.20 \\     
\hline                                   
\end{tabular}                                           
\caption{Average deviation (in mm) that combines both displacement and angular offset) between the simulated final pose and actual
final pose with 95 percent confidence interval for 3-point
support. The wrench-twist pairs used for training the model are
generated from the ideal limit surface. The average normalized error
for poly4 is 20.48\% and for quadratic is 24.97\%. }                                
\label{table:result_icra_pressure}        
\end{table}

\subsection{Pushing with stochasticity}
The experiment in \cite{yu2016more} demonstrates that the
same 2000 pushes in a highly controlled setting result in a
distribution of final poses. We perform simulations using the same
object and pusher geometry and push distance. The 2000 resultant trajectories and
histogram plot of pose changes are shown in Fig.
\ref{fig:single_push_traj} and \ref{fig:push_hist} respectively. 
We note that although the mean and variance pose changes are similar to the experiments with
abs material in \cite{yu2016more}, the distribution resemble a single Gaussian distribution which
differs from the multiple modes distribution in Figure 10 of
\cite{yu2016more}. We conjecture this is due to a time varying
stochastic process where coefficients of friction between surfaces
drift due to wear.

\begin{figure}[!h]
\centering
\begin{subfigure}[t]{\mylen}
\includegraphics[width=\mylen]{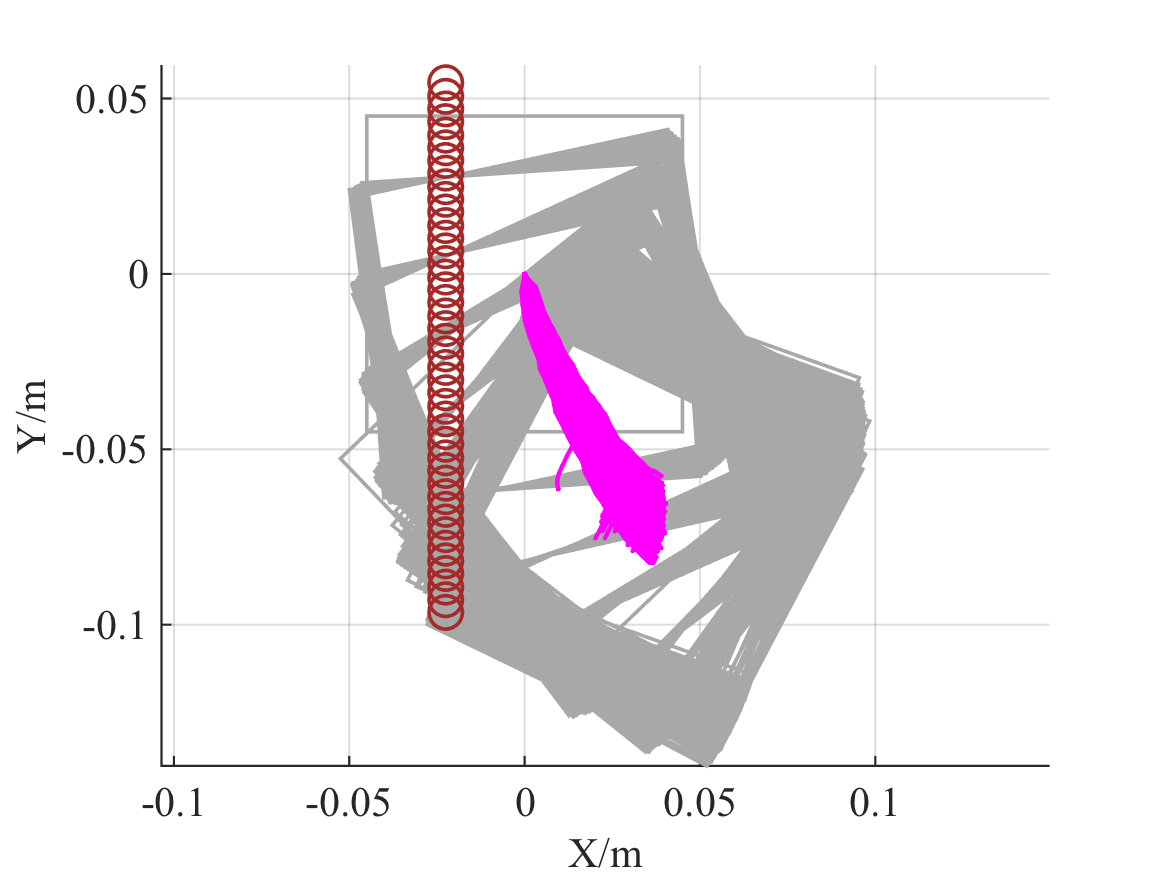} 
\caption{Simulation results.}
\end{subfigure}
\begin{subfigure}[t]{\mylen}
\includegraphics[width=\mylen]{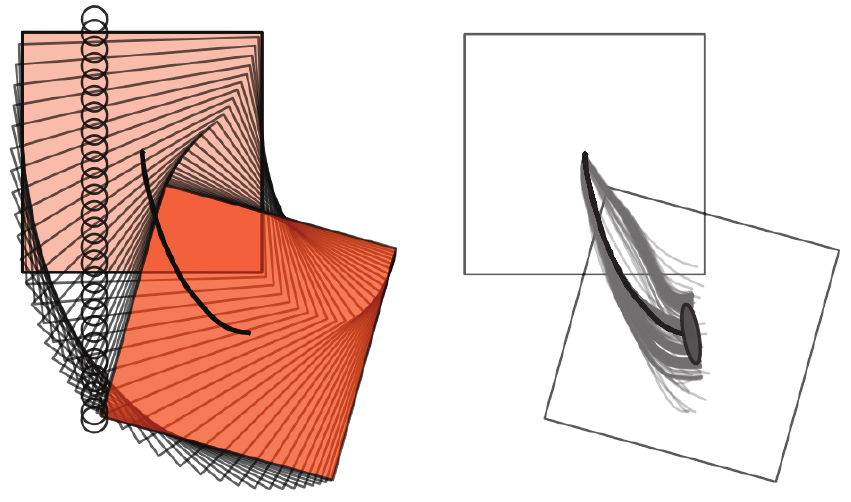}
\caption{Figure 9 of \cite{yu2016more}, reprinted with permission.}
\end{subfigure}
\caption{Stochastic modelling of single point pushing with fourth order sos-convex polynomial
  representation of the limit surface using wrench twist pairs generated
  from 64 grids with uniform pressure. The degree of
  freedom in the sampling distribution equals 20. The
  contact coefficient of friction between the pusher and the object is
  uniformly sampled from 0.15 to 0.35. The trajectories are
  qualitatively similar to the experimental results in Figure 9 of \cite{yu2016more}. \\ }
\label{fig:single_push_traj}
\vspace{-0.2in}                      
\end{figure}
\begin{figure}[!h]
\centering
\begin{subfigure}[t]{\mylenshort}
\centering
\includegraphics[width=\mylenshort]{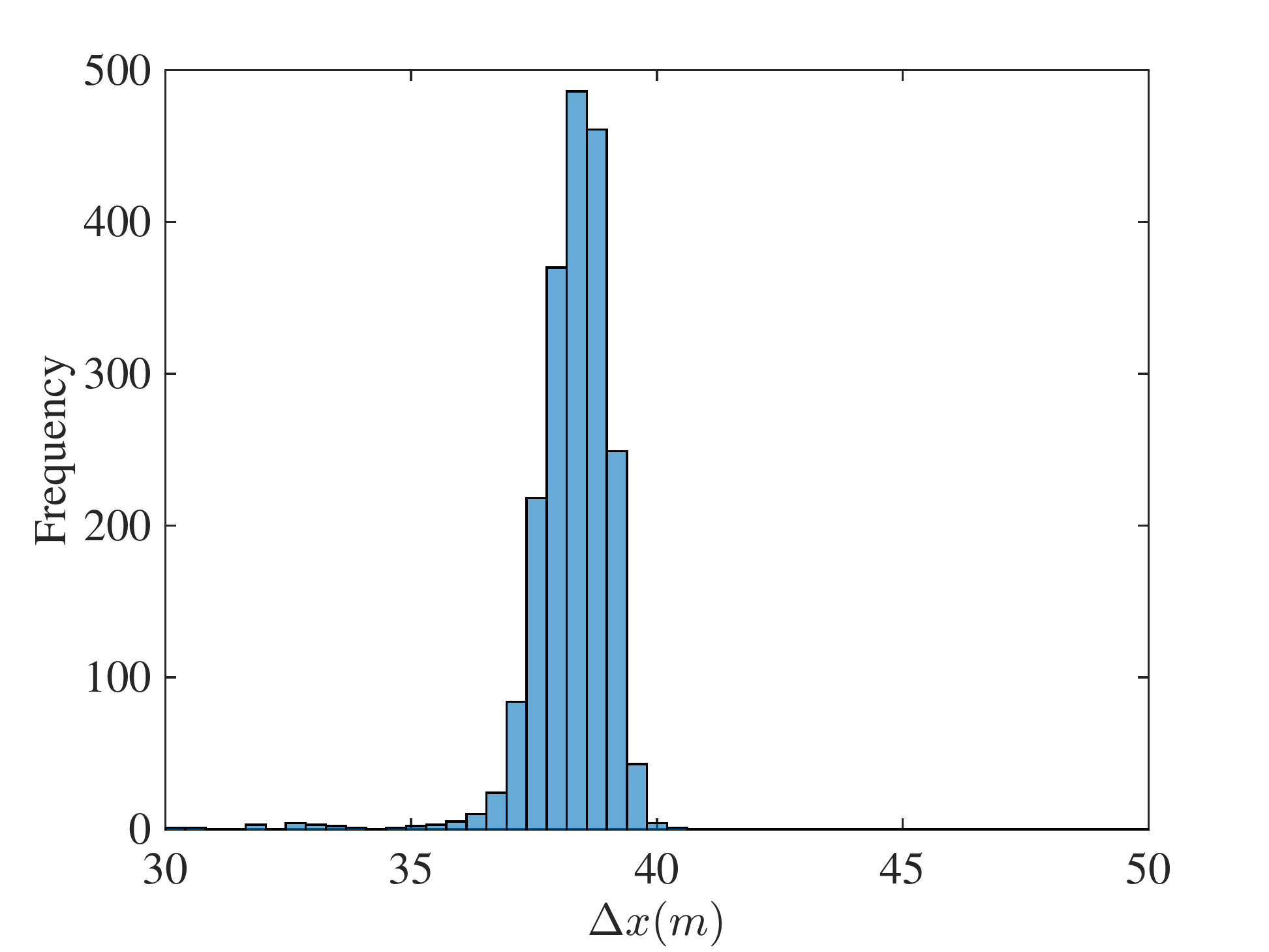} 
\caption{Histogram of $\Delta x$ \\ }
\end{subfigure}
\hspace{-1.8em}
~
\begin{subfigure}[t]{\mylenshort}
\centering
\includegraphics[width=\mylenshort]{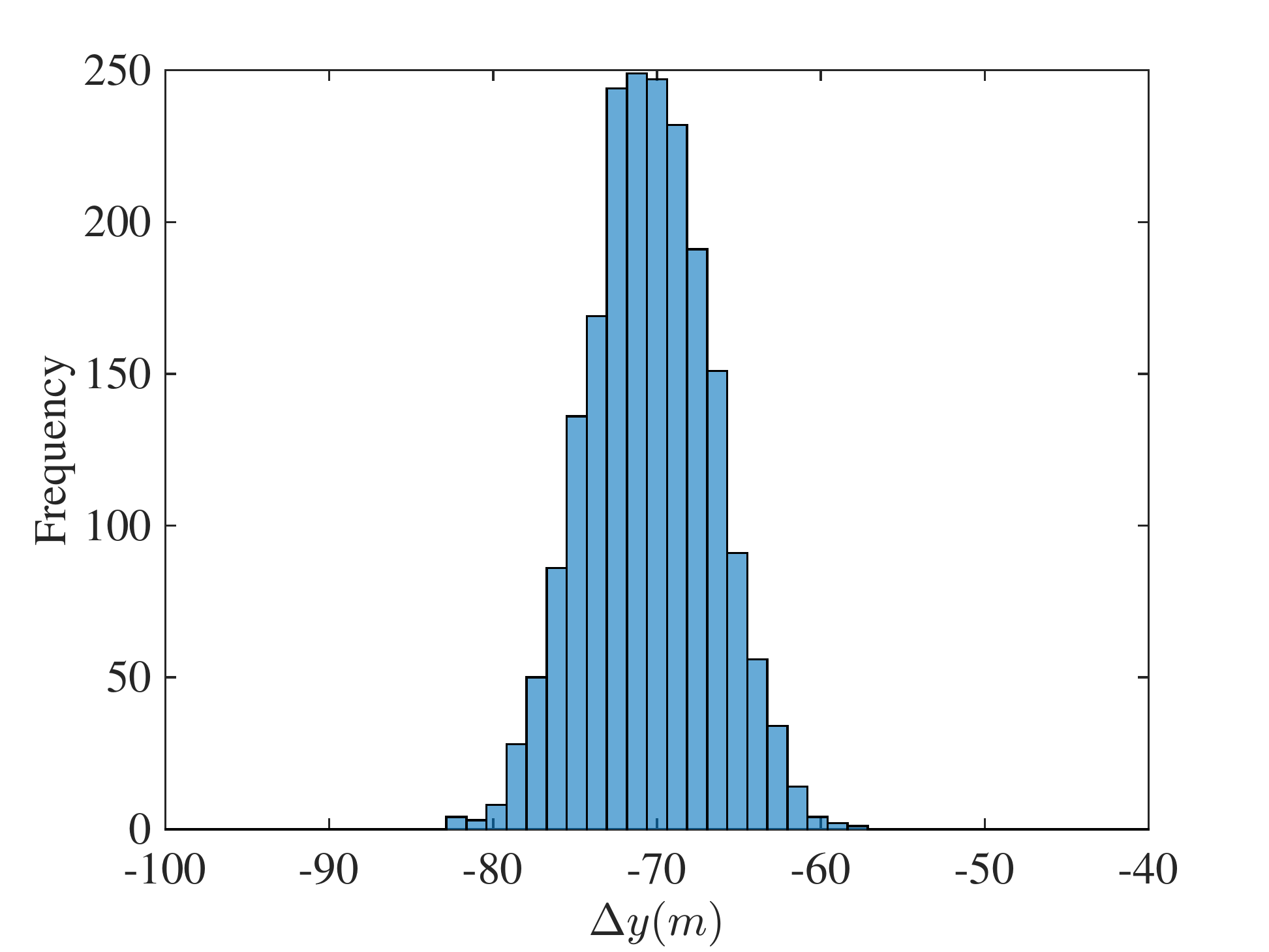}
\caption{Histogram of $\Delta y$ \\}
\end{subfigure} 
\hspace{-1.8em}
~
\begin{subfigure}[t]{\mylenshort}
\centering
\includegraphics[width=\mylenshort]{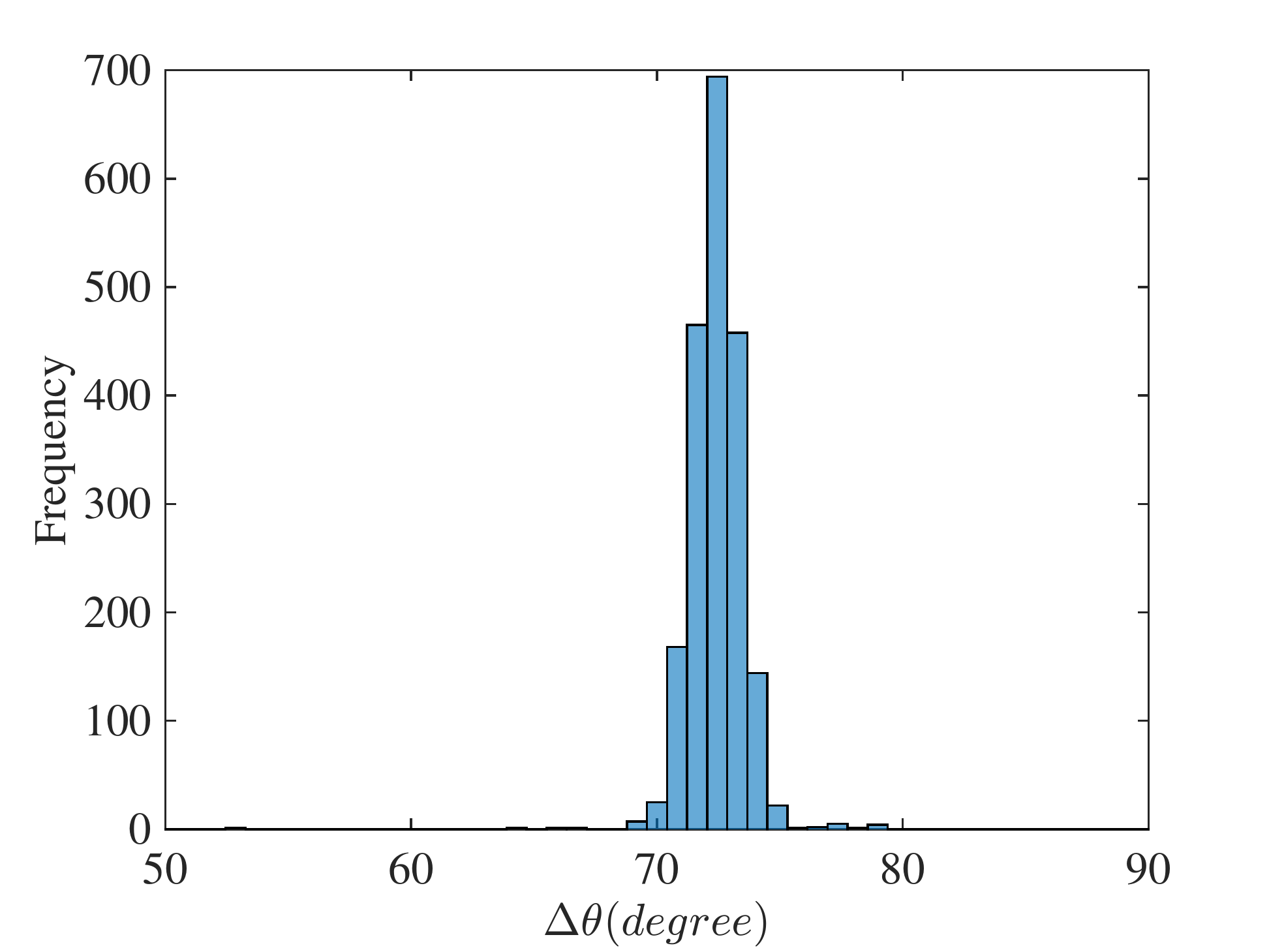}
\caption{Histogram of $\Delta \theta$ }
\end{subfigure}
\caption{Histogram of final poses for the 2000 pushing trajectories. The object
  is initially at (0,0,0). Note the curve resembles a Gaussian distribution.}
\label{fig:push_hist}
\vspace{-0.15in}                      
\end{figure}

We also evaluate the effects of uncertainty reduction with 2 point
fingers under the stochastic contact model. The circular object has
radius of 5.25cm. The two fingers separated by 10cm perform a
straight line push of 26.25cm. The desired goal is to have the object
centered with respect to the two fingers.
Fig. \ref{fig:twopoint_push1} and \ref{fig:twopoint_push2} compare the
resultant trajectories under different amount of system noise. We find
that despite larger noise in the resultant trajectories, the convergent
region of the stable goal pose differs by less than 5\% and the
difference is mostly around the uncertainty boundary. A kernel density
plot of the convergence region is shown in in
Fig. \ref{fig:twopoint_push_capture} for $n_{df} = 10$. We conclude
that multiple active constraints induce a large region of attraction.
\begin{figure*}[!h]
\centering
\begin{subfigure}[t]{\mylenmedium}
\centering
\includegraphics[width=\mylenmedium]{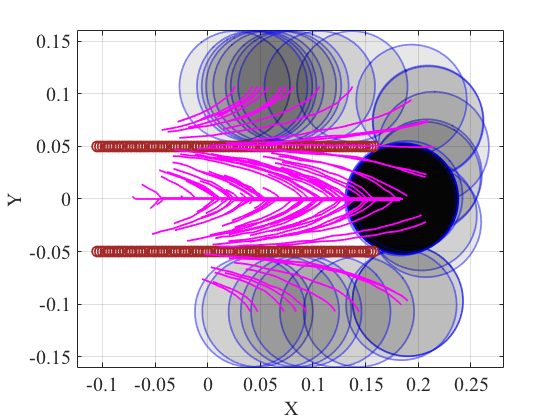} 
\caption{100 pushed trajectories of different initial poses using ellipsoid
  representation of $H(\mathbf{F})$ with $n_{df} = 200$.\\}
\label{fig:twopoint_push1}
\vspace{-0.1in}
\end{subfigure}
~
\begin{subfigure}[t]{\mylenmedium}
\centering
\includegraphics[width=\mylenmedium]{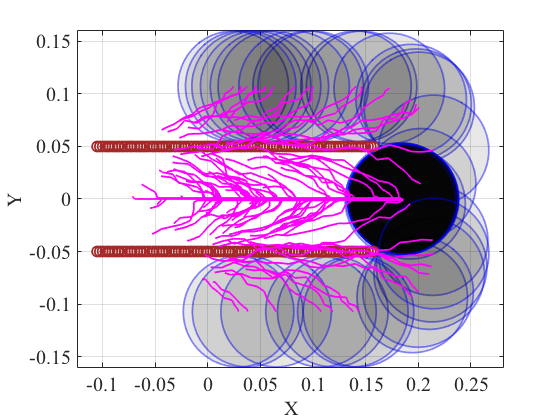} 
\caption{100 pushed trajectories of different initial poses using ellipsoid
  representation of $H(\mathbf{F})$ with $n_{df} = 10$.\\}
\label{fig:twopoint_push2}
\vspace{-0.1in}
\end{subfigure}
~
\begin{subfigure}[t]{\mylenmedium}
\centering
\includegraphics[width=\mylenmedium]{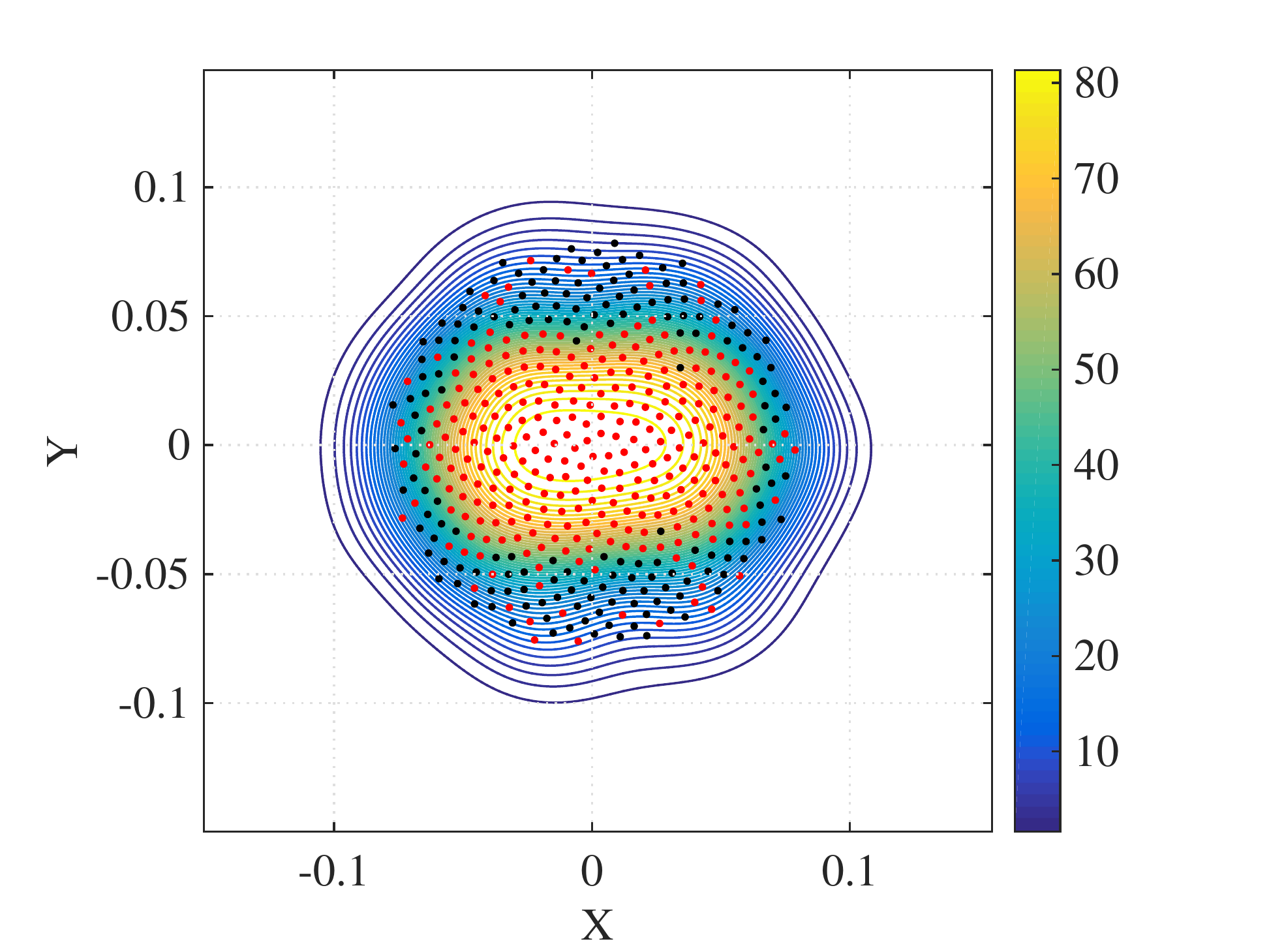}
\caption{Kernel density plot of the convergence region for $n_{df} =
  10$. Convergent initial poses are in red, and the rest are in black. \\}
\label{fig:twopoint_push_capture}
\vspace{-0.1in}
\end{subfigure}
\caption{Simulation results using the proposed
 contact model illustrating the process of two point fingers
  pushing a circle to reduce uncertainty. A total of 500 initial
  object center positions are uniformly sampled from a circle of
  radius 7.88cm. }
\label{fig:twopointpush}
\vspace{-0.1in}
\end{figure*}

\subsection{Grasping under uncertainty} \label{sec:squeeze}
We conduct robotic experiments to evaluate our contact model for
grasping. Fig. \ref{fig:rect_pressure} shows two rectangular objects
with the same geometry but different pressure
distributions. Another experiment is conducted for a butterfly shaped
object shown in Fig. \ref{fig:butterfly}. We use the Robotiq hand \cite{robotiq2016} 
and represent it as a planar parallel-jaw gripper with rectanglular
fingers as shown in Fig. \ref{fig:prev_grasp_rect} and Fig. \ref{fig:prev_grasp_butter}.
Convex quadratic limit surface parameterizations $H(\mathbf{F})$ are trained from
wrench-twist pairs from a uniform friction distribution along the
object boundary. The sampling degree of freedom $n_{df}$ equals 250
with contact friction coefficient $\mu_c$ sampled uniformly from
[0.015, 0.02]. The simulated results with the stochastic contact
model match well with the rectangles for both pressure
distribution. However, the model fails to capture the stability of
grasps and the deformation of objects. In the case of a butterfly-shaped
object, many unstable grasps and jamming equilbria exist, but as the fingers increase
the gripping force the object will ``fly'' away as the stored elastic
energy turns into large accelerations which violates the
quasistatic assumptions of our model, as revealed in the scattered
post-grasp distribution in Fig. \ref{fig:postposes_butter}. We also
compare the cases where dynamics do not play a major role:
Fig. \ref{fig:postposes_zoomed_butter} shows the zoomed in plots to
compare with simulation results in
Fig. \ref{fig:postposes_butter_sim}. We can see that the
model simulation deviates more compared with the case for rectangular
geometry. Comparing the histogram plot in
Fig. \ref{fig:hist_butter_sim} and Fig. \ref{fig:hist_butter_zoom},
we can see that the simulation returns more jamming and grasping final
states as illustrated by the spikes in $\theta$. 

\begin{figure*}[h!]
\centering
\begin{subfigure}[t]{\figbarlen}
\centering
\includegraphics[width=\figbarlen]{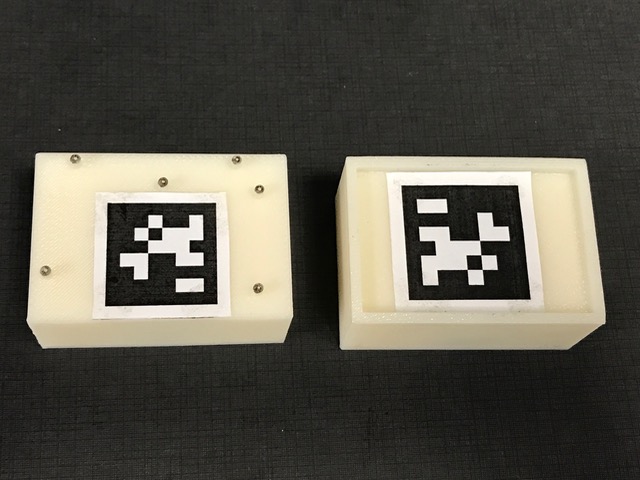} 
\setcounter{subfigure}{0}%
\caption{Two 50mm $\times$ 35mm rectangles with 6 points and boundary pressure distribution. }
\label{fig:rect_pressure}
\end{subfigure}
~
\begin{subfigure}[t]{\figbarlen}
\centering
\includegraphics[width=\figbarlen]{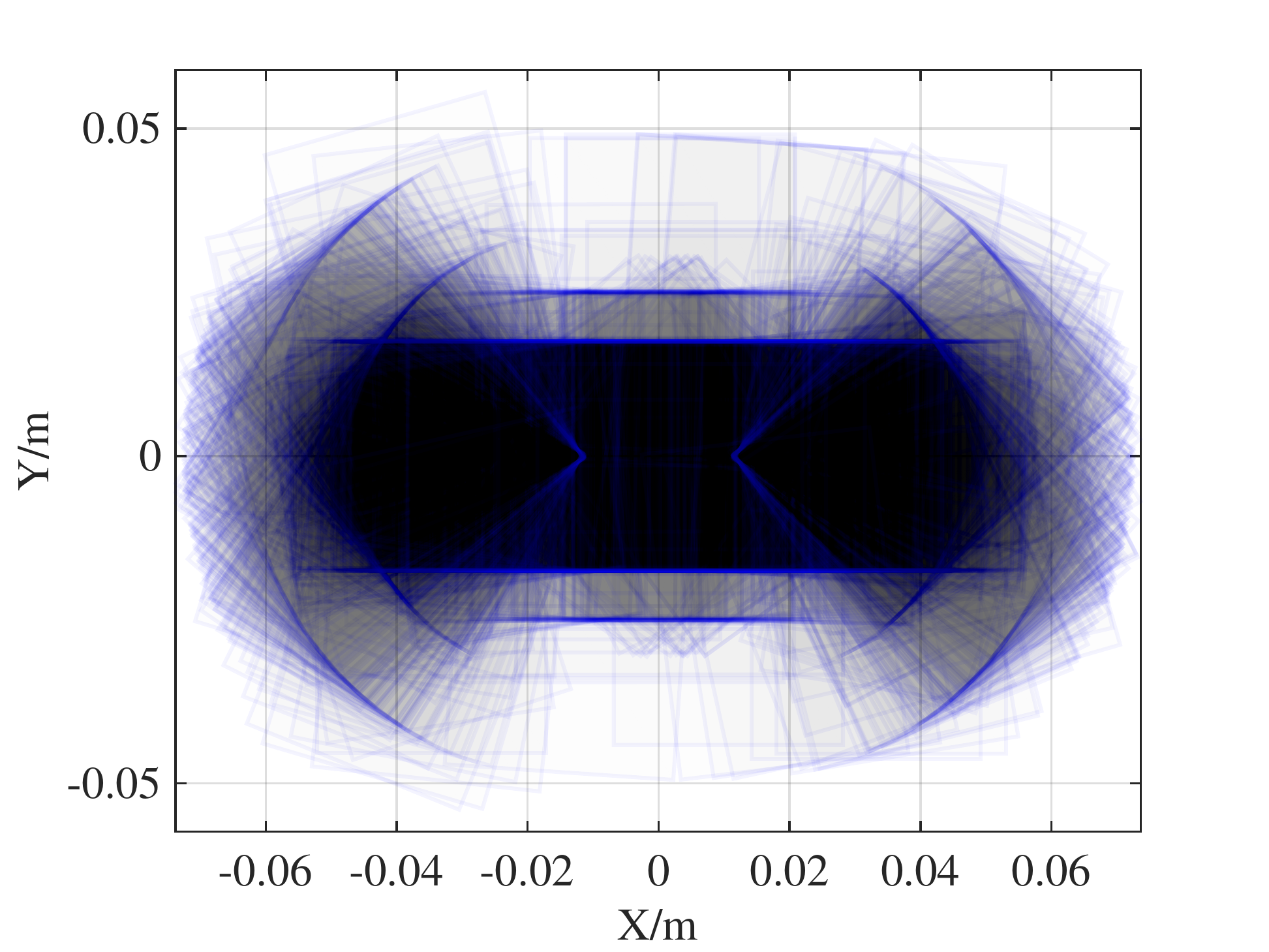} 
\setcounter{subfigure}{2}%
\caption{Distribution of the simulated post-grasp poses using the stochastic contact model.}
\label{fig:rect_post_grasp_sim}
\end{subfigure}
~
\begin{subfigure}[t]{\figbarlen}
\centering
\includegraphics[width=\figbarlen]{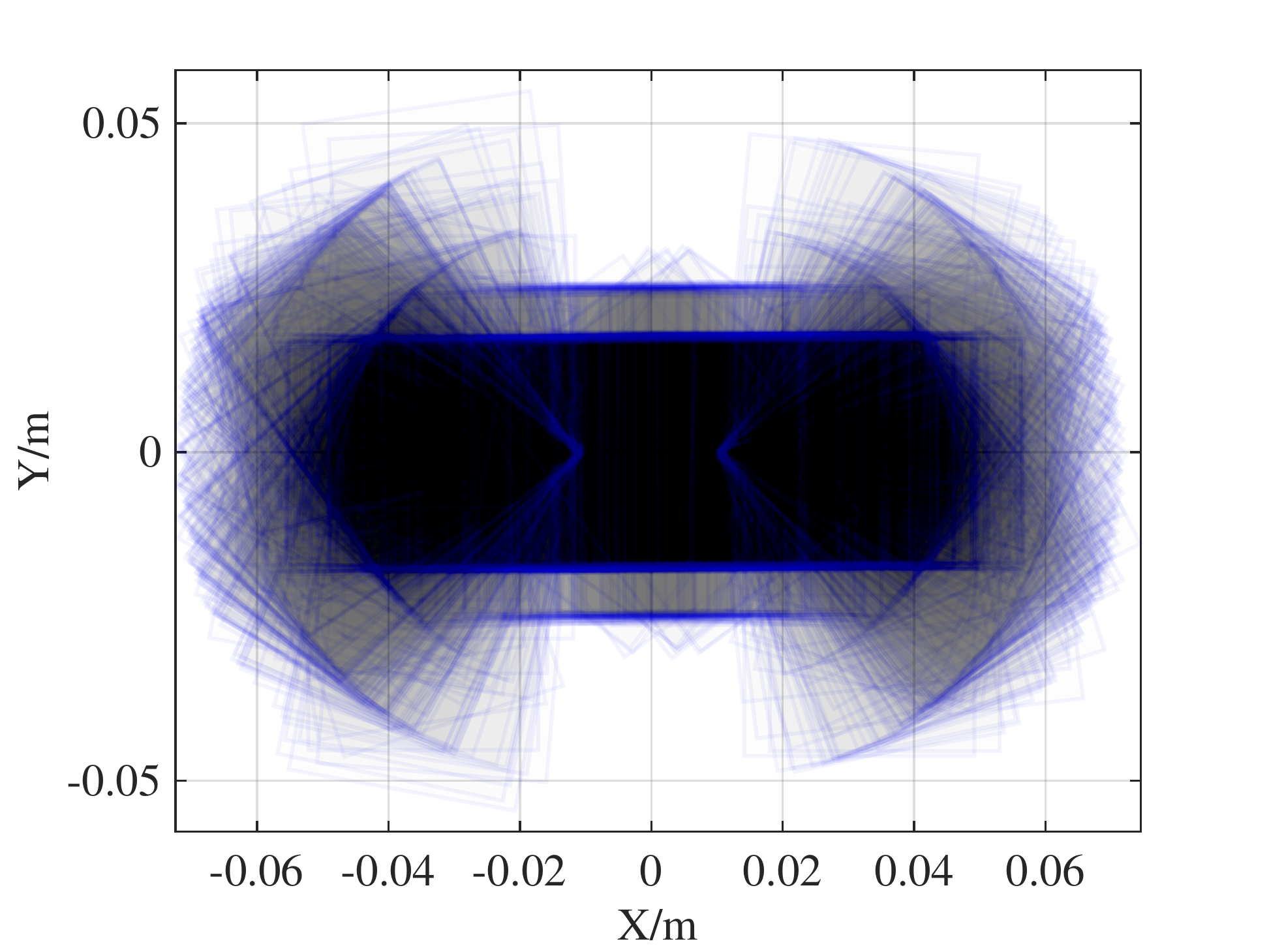} 
\setcounter{subfigure}{4}%
\caption{Distribution of the experimental post-grasp poses for
  the boundary pressure.}
\label{fig:rect_post_grasp_rim}
\end{subfigure}
~
\begin{subfigure}[t]{\figbarlen}
\centering
\includegraphics[width=\figbarlen]{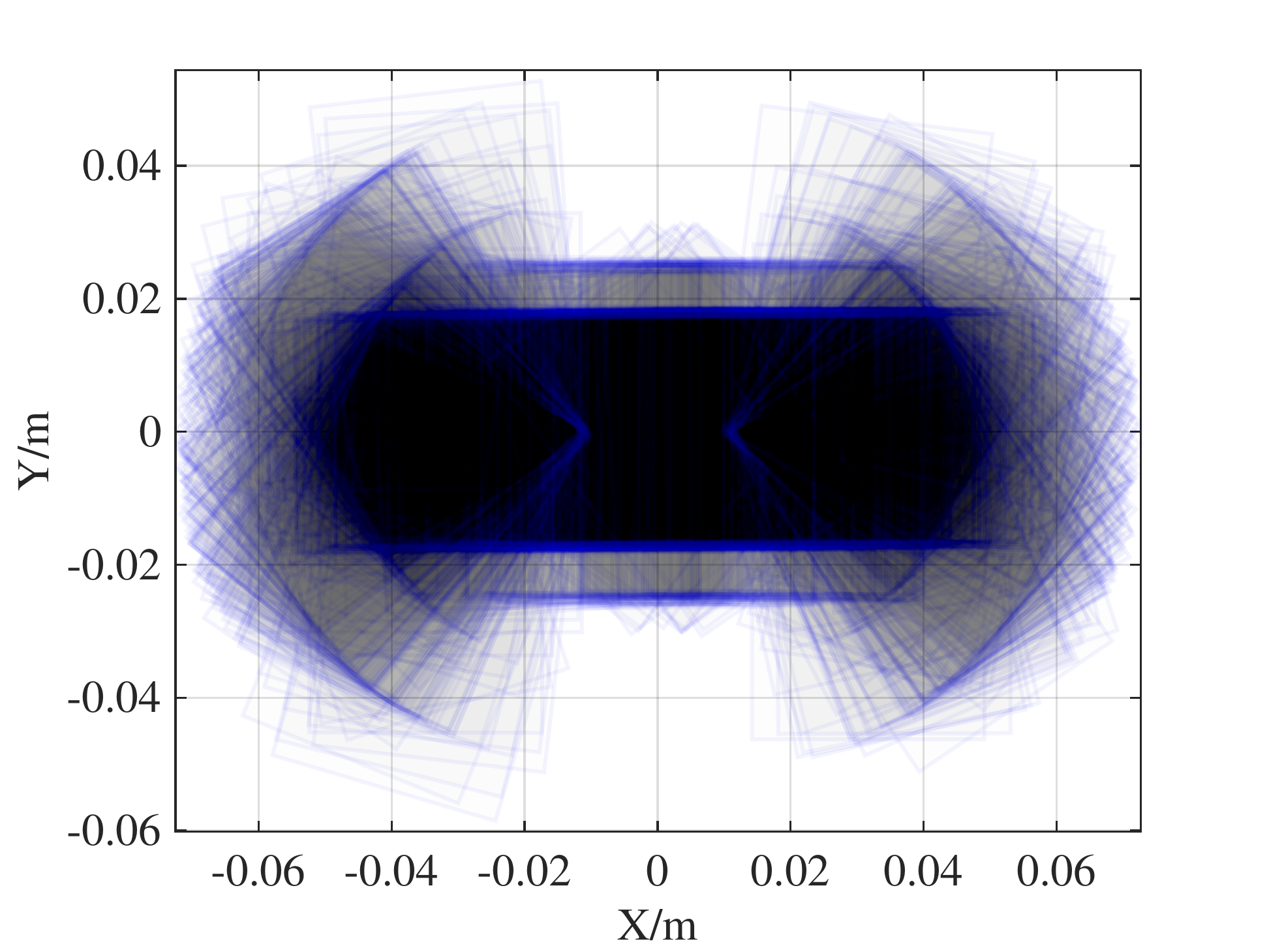} 
\setcounter{subfigure}{6}%
\caption{Distribution of the experimental post-grasp poses for
  the 6-points pressure.}
\label{fig:rect_post_grasp_6pts}
\end{subfigure}
~
\begin{subfigure}[t]{\figbarlen}
\centering
\includegraphics[width=\figbarlen]{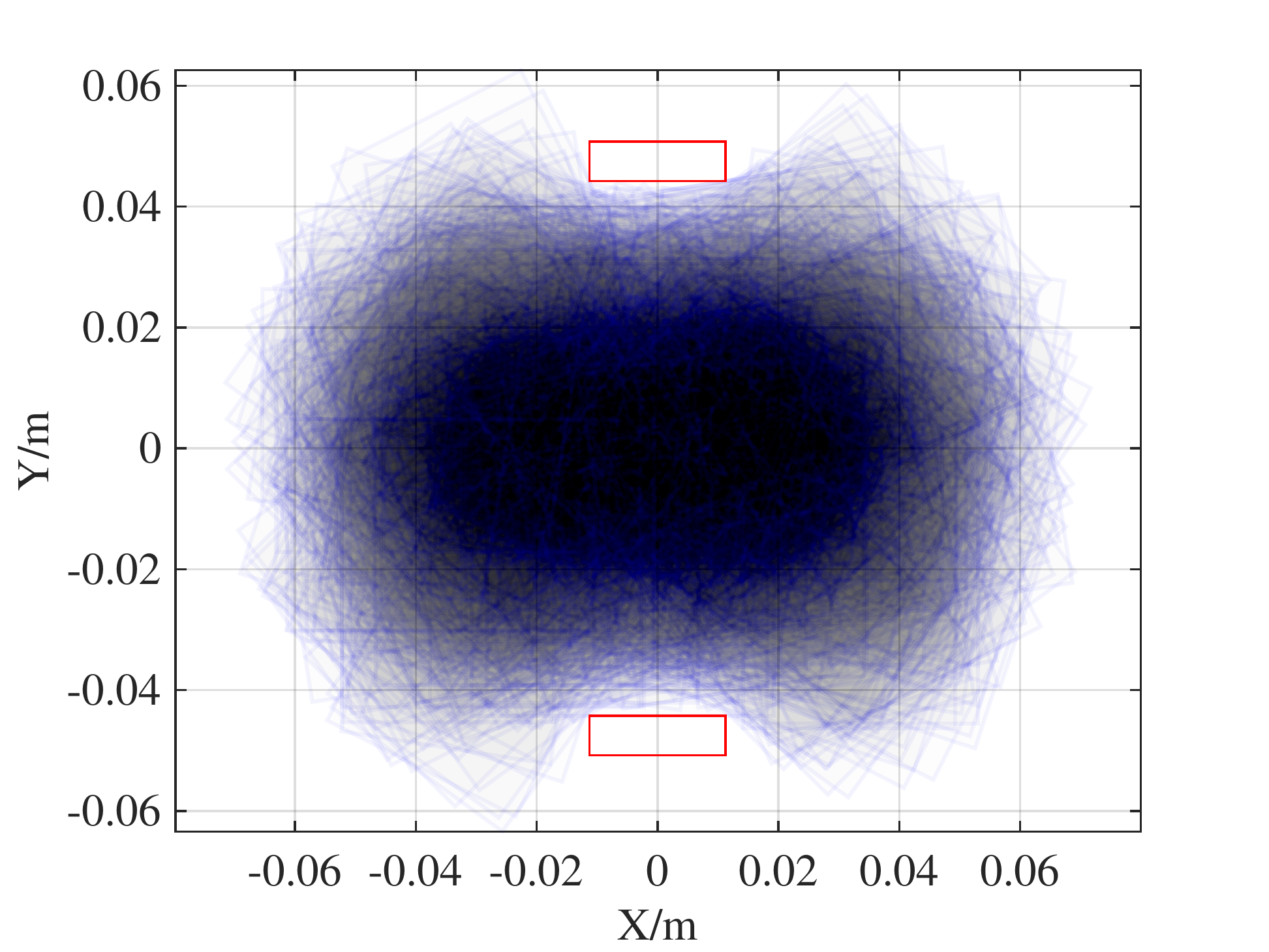} 
\setcounter{subfigure}{1}%
\caption{Initial uncertainty of 600 poses. }
\label{fig:prev_grasp_rect}
\end{subfigure}
~
\begin{subfigure}[t]{\figbarlen}
\centering
\includegraphics[width=\figbarlen]{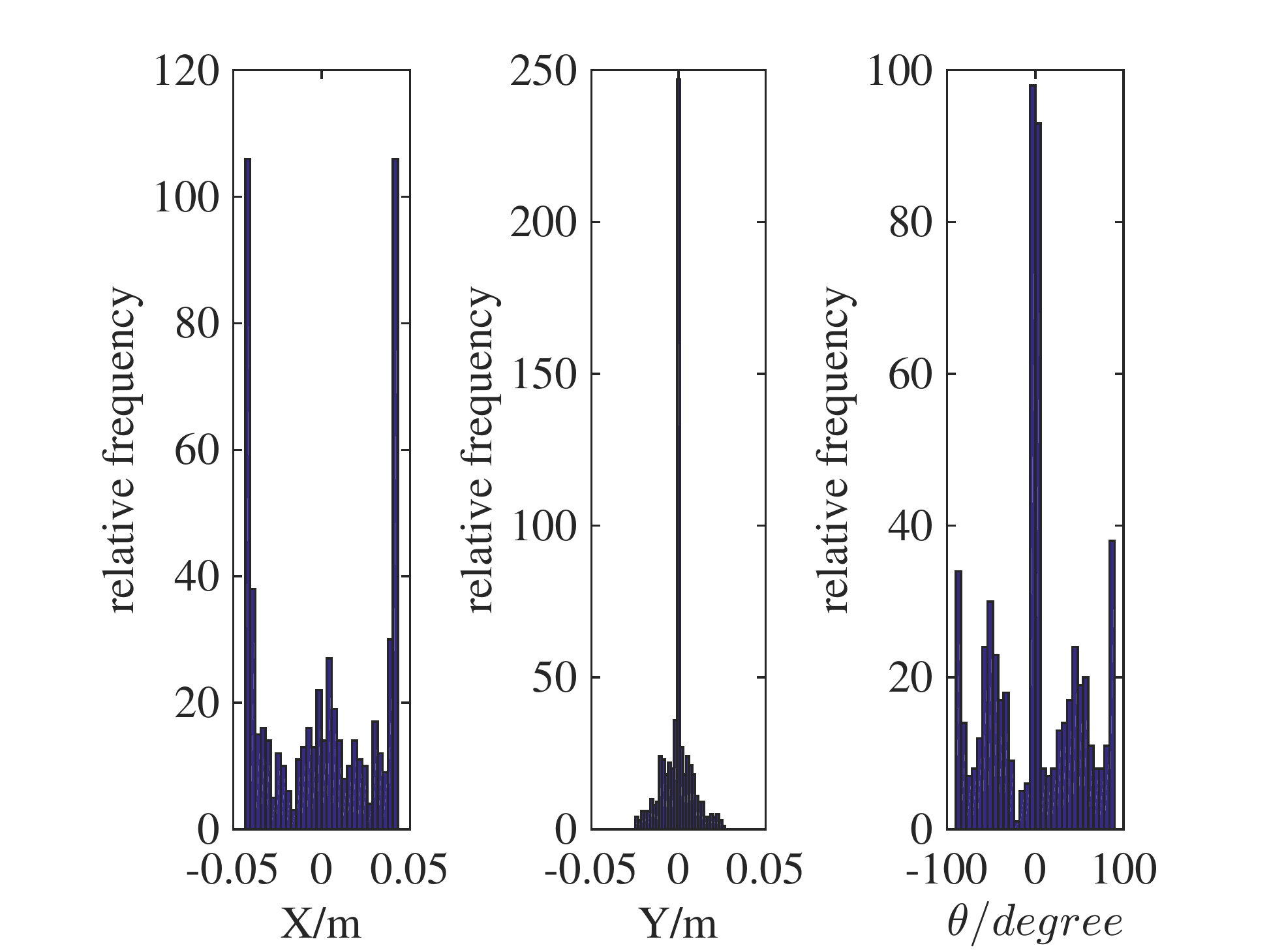}
\setcounter{subfigure}{3}%
\caption{Histogram plot of the simulated post distribution.}
\label{fig:wedge_grasp}
\end{subfigure}
~
\begin{subfigure}[t]{\figbarlen}
\centering
\includegraphics[width=\figbarlen]{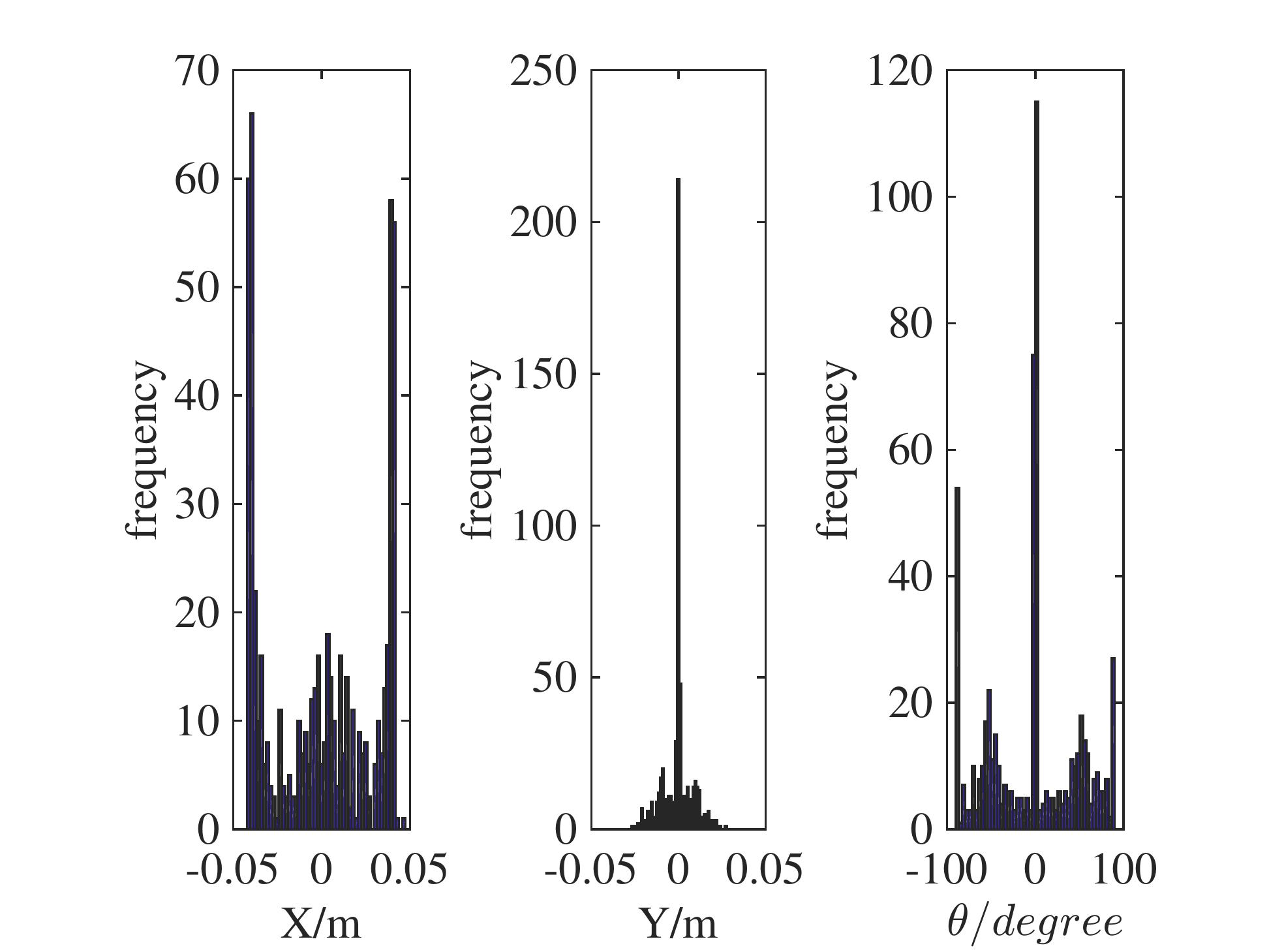}
\setcounter{subfigure}{5}%
\caption{Histogram of the experimental post distribution for the
  boundary pressure.}
\label{fig:slip_grasp}
\end{subfigure}
~
\begin{subfigure}[t]{\figbarlen}
\centering
\includegraphics[width=\figbarlen]{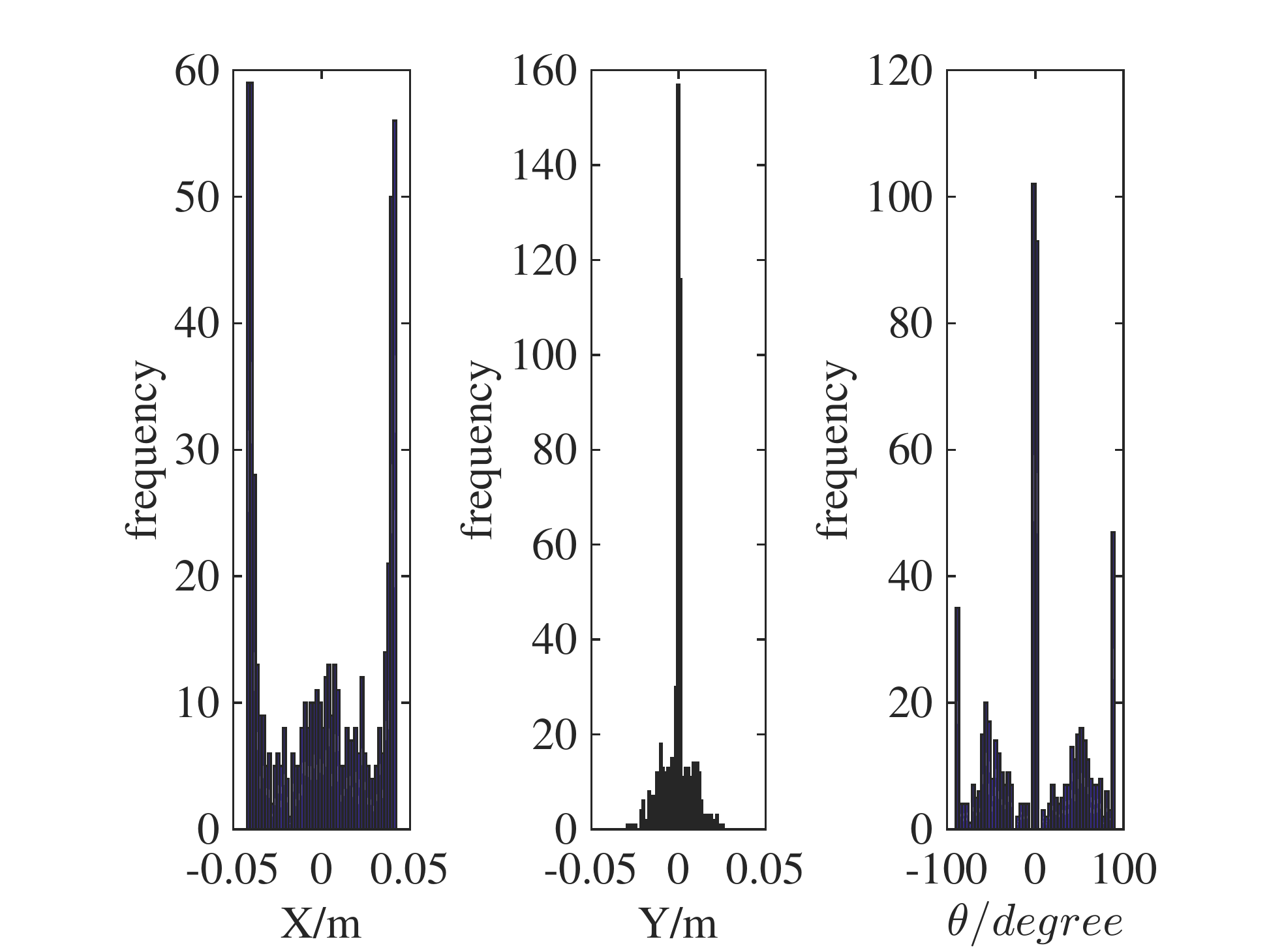}
\setcounter{subfigure}{7}%
\caption{Histogram of the experimental post distribution for the
  6-points pressure.}
\label{fig:hist_6pts}
\end{subfigure}
\caption{Experiments on the rectangular objects with different pressure
  distributions. 600 initial poses are sampled whose centers are
  uniformly distributed in a circle of radius of 20mm and angles are uniformly distributed from -90 to 90 degrees. }
\end{figure*}

\begin{figure*}[h!]
\centering
\begin{subfigure}[t]{\figbarlen}
\centering
\includegraphics[width=\figbarlen]{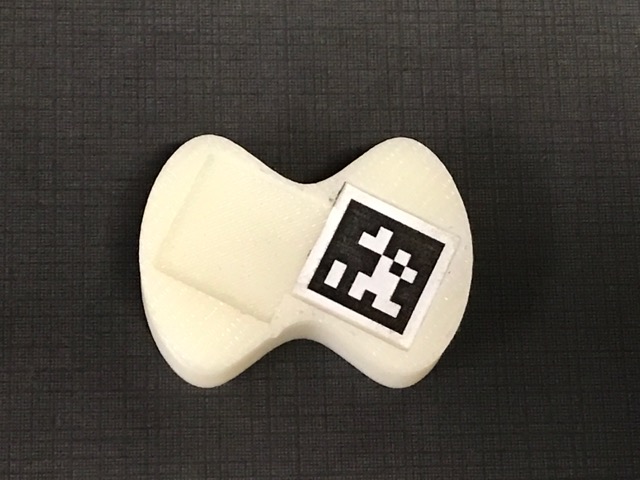} 
\setcounter{subfigure}{0}%
\caption{Butterfly shaped object with boundary pressure distribution
  used for experiment.}
\label{fig:butterfly}
\end{subfigure}
~
\begin{subfigure}[t]{\figbarlen}
\centering
\includegraphics[width=\figbarlen]{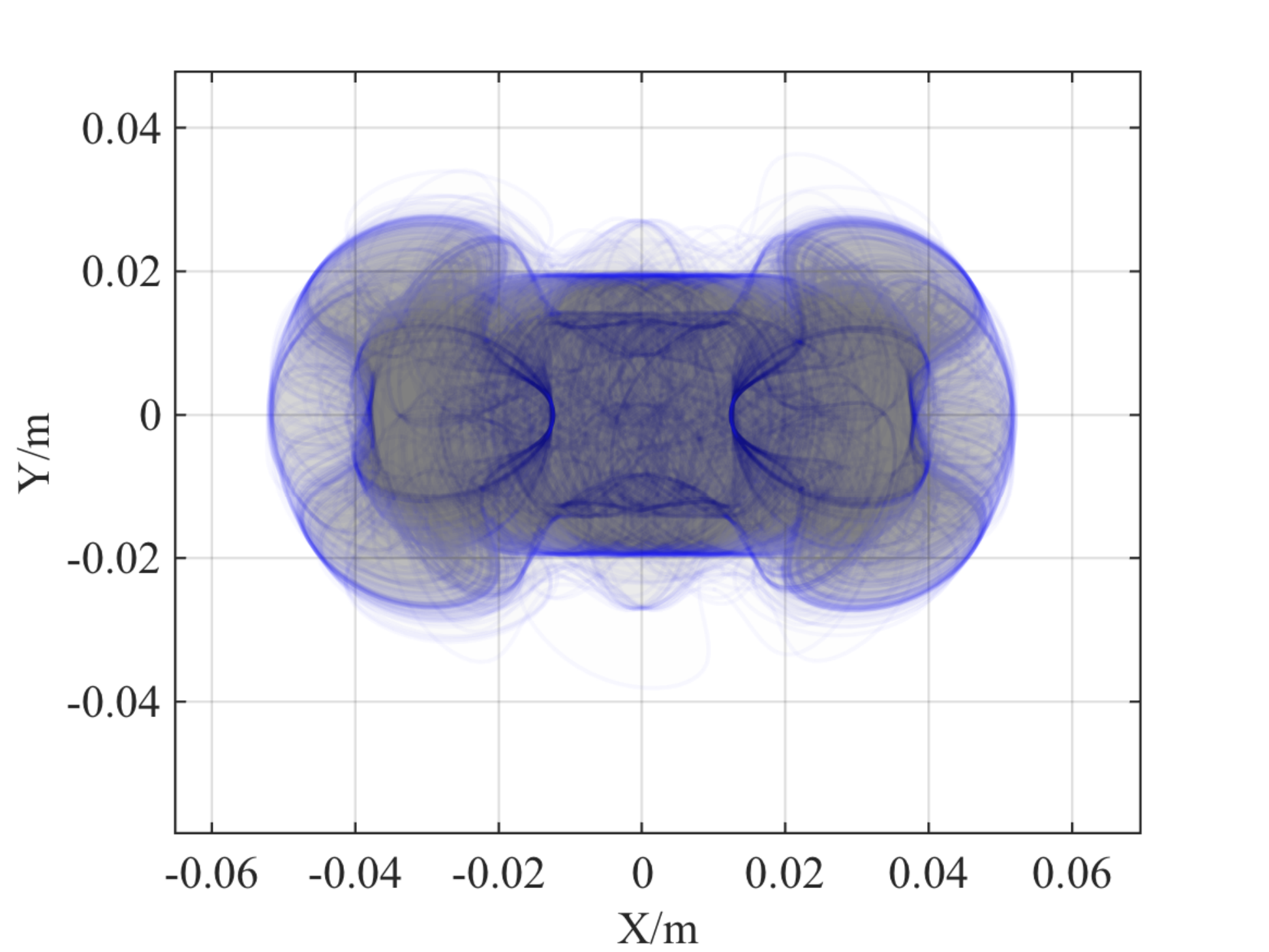} 
\setcounter{subfigure}{2}%
\caption{Distribution of the simulated post-grasp poses using the stochastic contact model.}
\label{fig:postposes_butter_sim}
\end{subfigure}
~
\begin{subfigure}[t]{\figbarlen}
\centering
\includegraphics[width=\figbarlen]{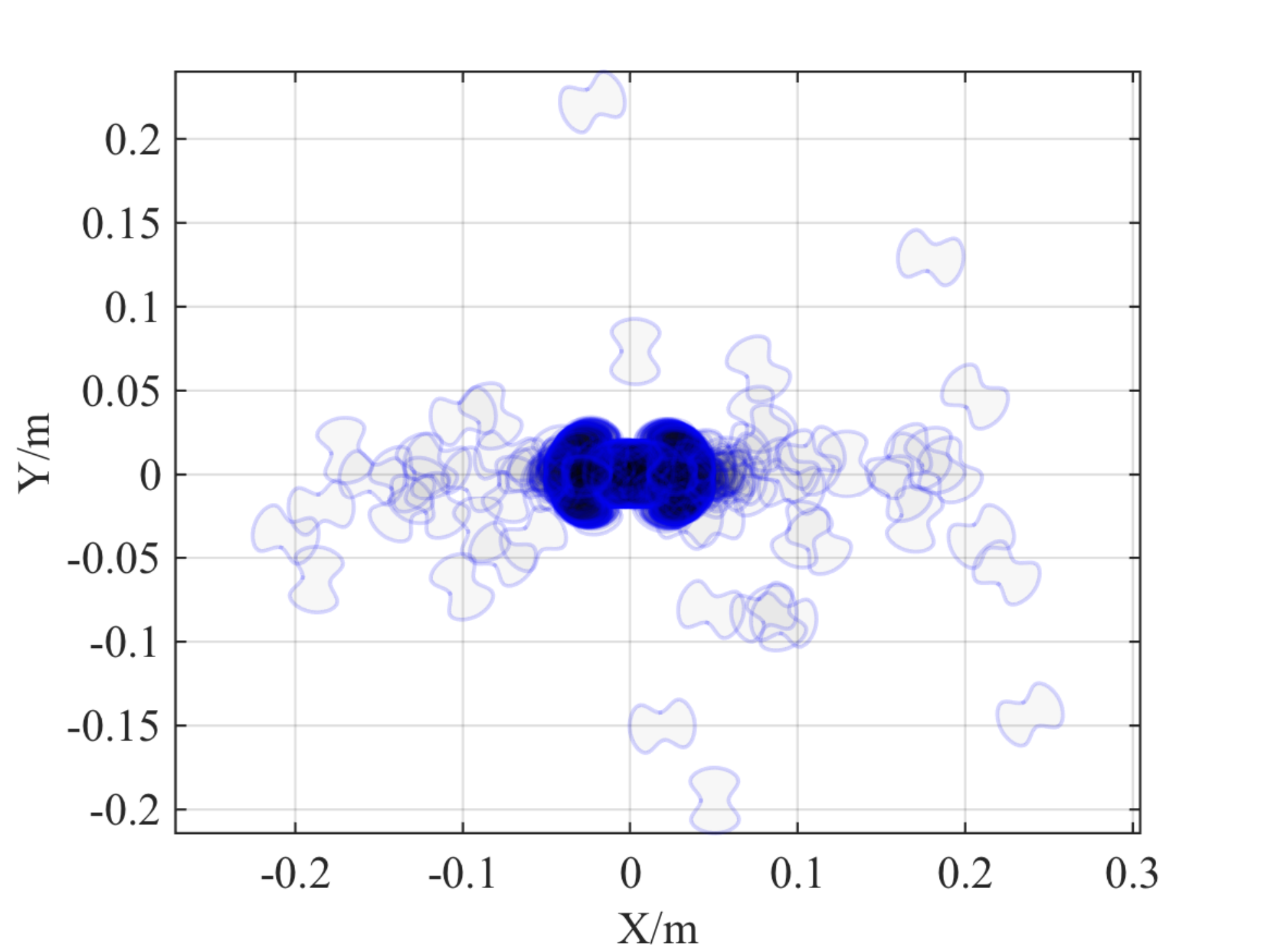} 
\setcounter{subfigure}{4}%
\caption{Distribution of the object poses after the grasping
  actions from experimental data.}
\label{fig:postposes_butter}
\end{subfigure}
~
\begin{subfigure}[t]{\figbarlen}
\centering
\includegraphics[width=\figbarlen]{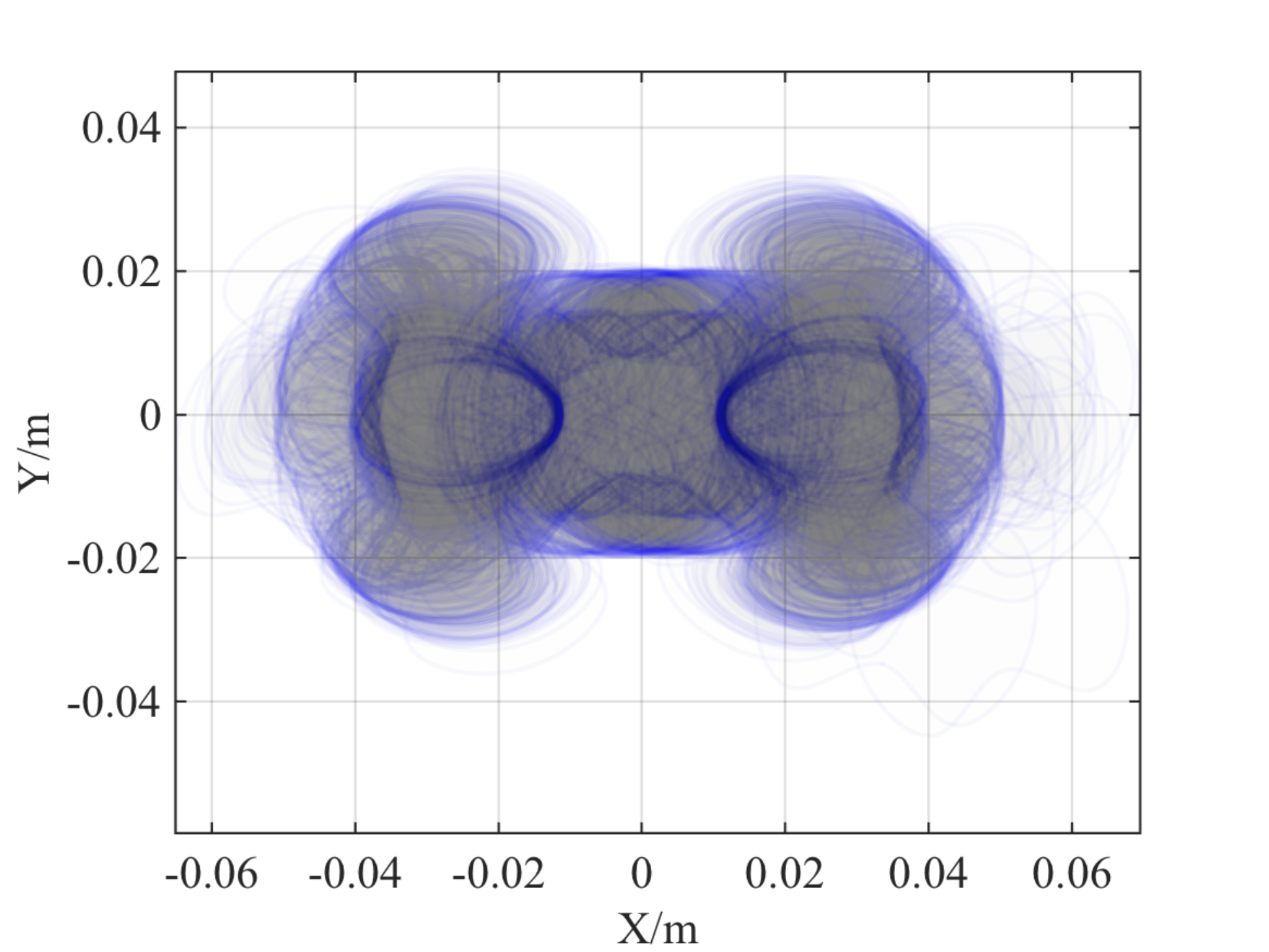} 
\setcounter{subfigure}{6}%
\caption{Zoomed-in distribution of the object
  poses after the grasping actions around the origin. }
\label{fig:postposes_zoomed_butter}
\end{subfigure}
~
\begin{subfigure}[t]{\figbarlen}
\centering
\includegraphics[width=\figbarlen]{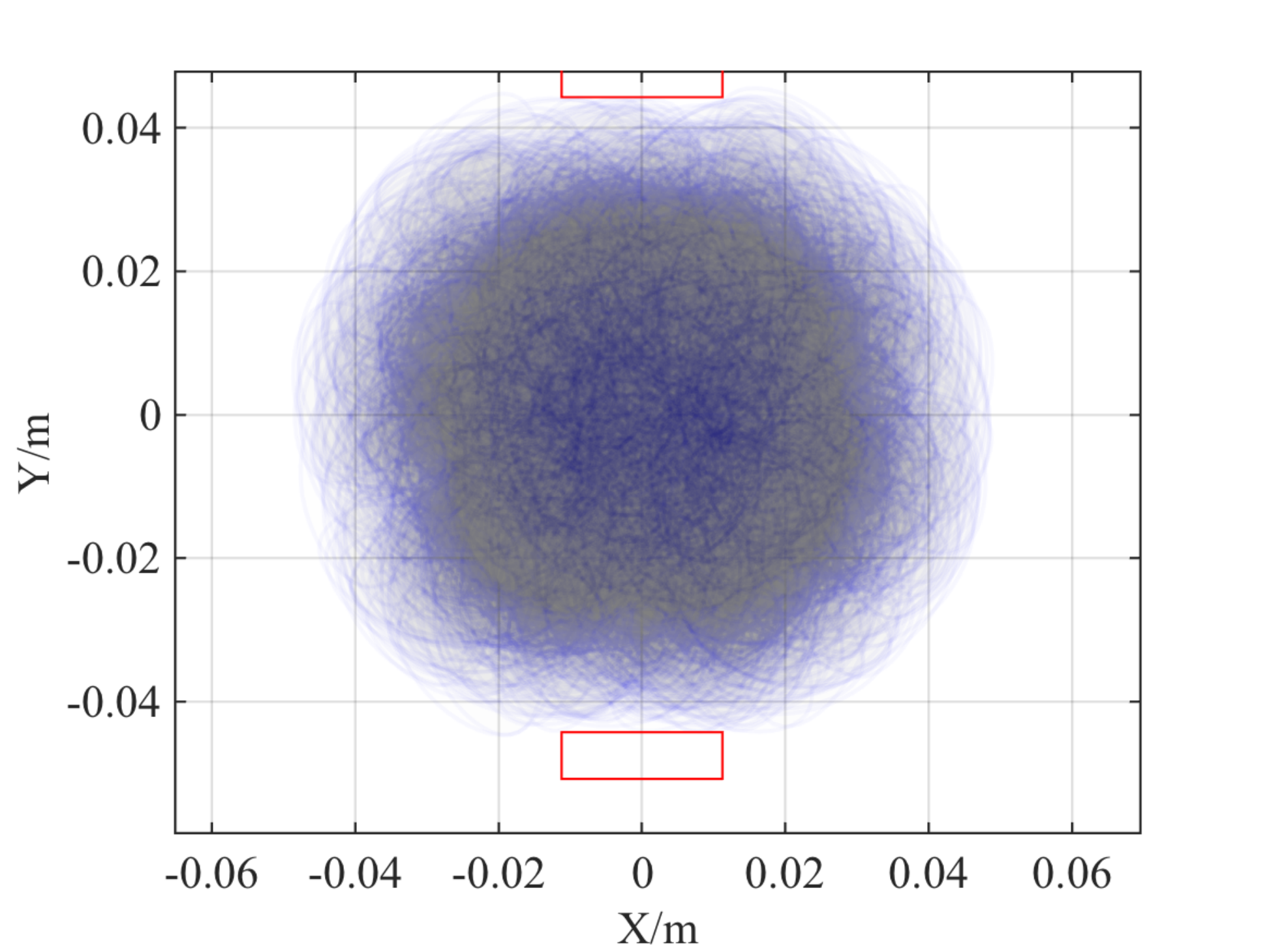} 
\setcounter{subfigure}{1}%
\caption{Initial uncertainty of 900 poses.}
\label{fig:prev_grasp_butter}
\end{subfigure}
~
\begin{subfigure}[t]{\figbarlen}
\centering
\includegraphics[width=\figbarlen]{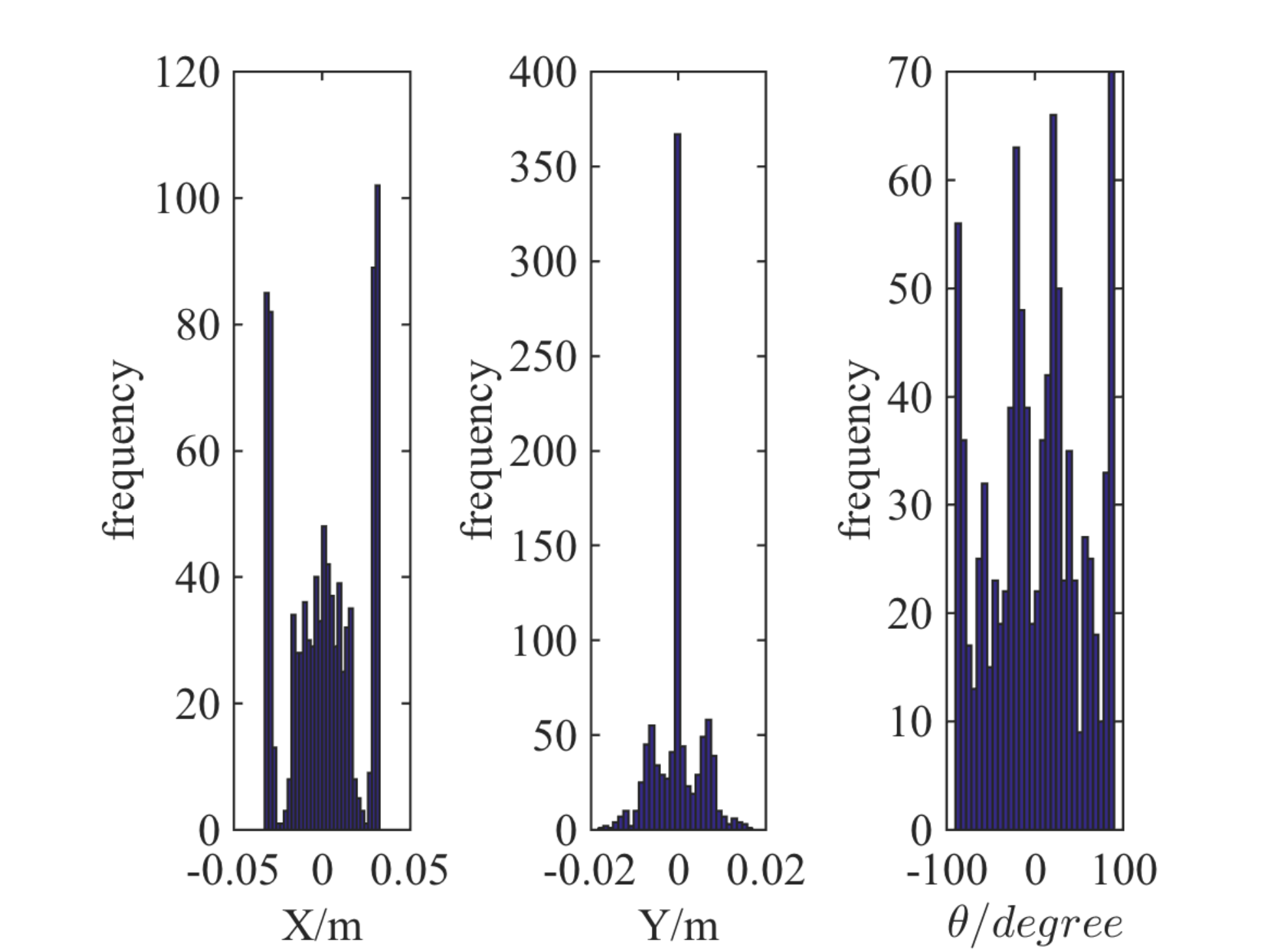}
\setcounter{subfigure}{3}%
\caption{Histogram plot for the simulated post distribution.}
\label{fig:hist_butter_sim}
\end{subfigure}
~
\begin{subfigure}[t]{\figbarlen}
\centering
\includegraphics[width=\figbarlen]{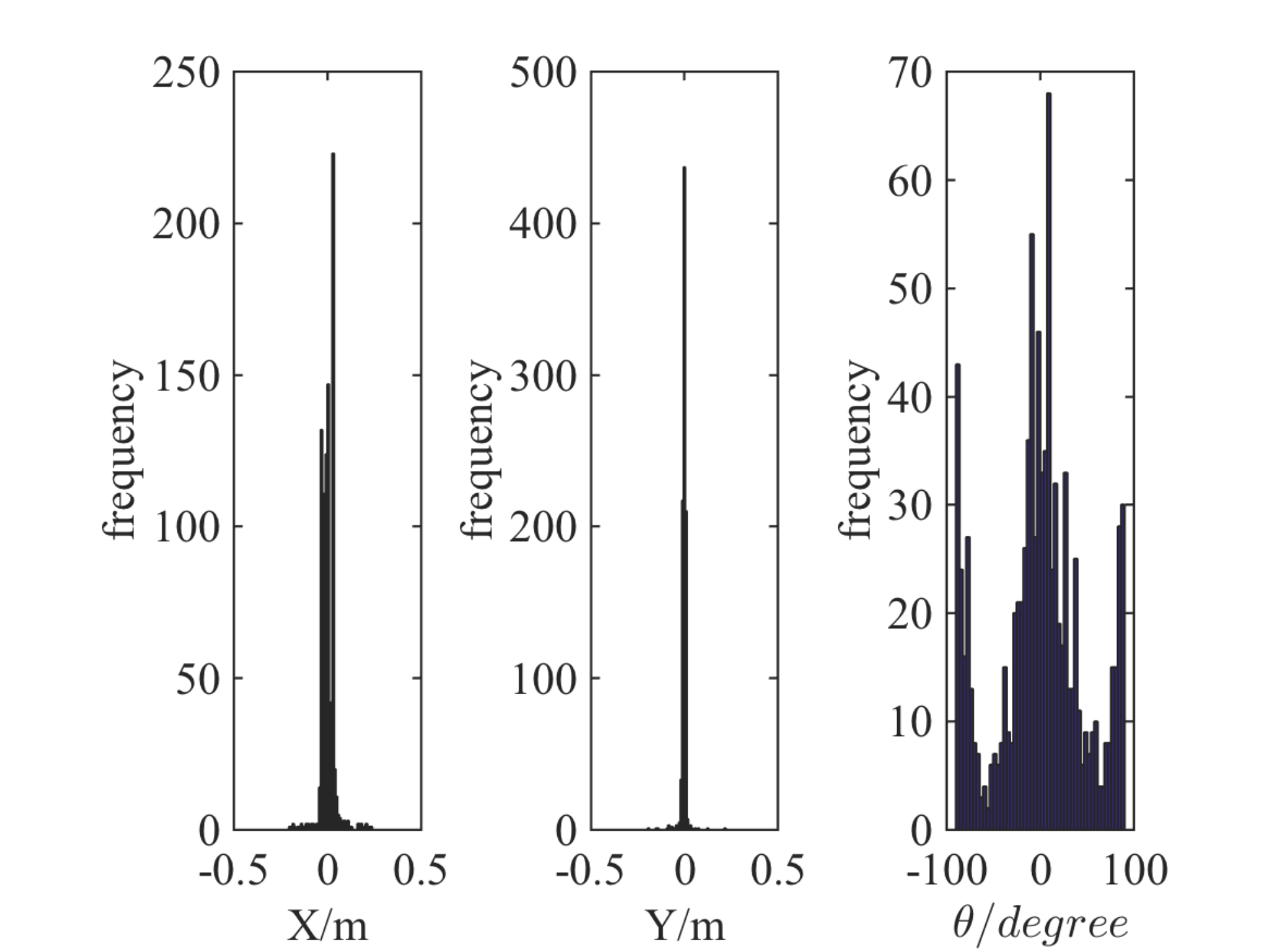}
\setcounter{subfigure}{5}%
\caption{Histogram plot of the experimental post distribution.}
\label{fig:hist_butter}
\end{subfigure}
~
\begin{subfigure}[t]{\figbarlen}
\centering
\includegraphics[width=\figbarlen]{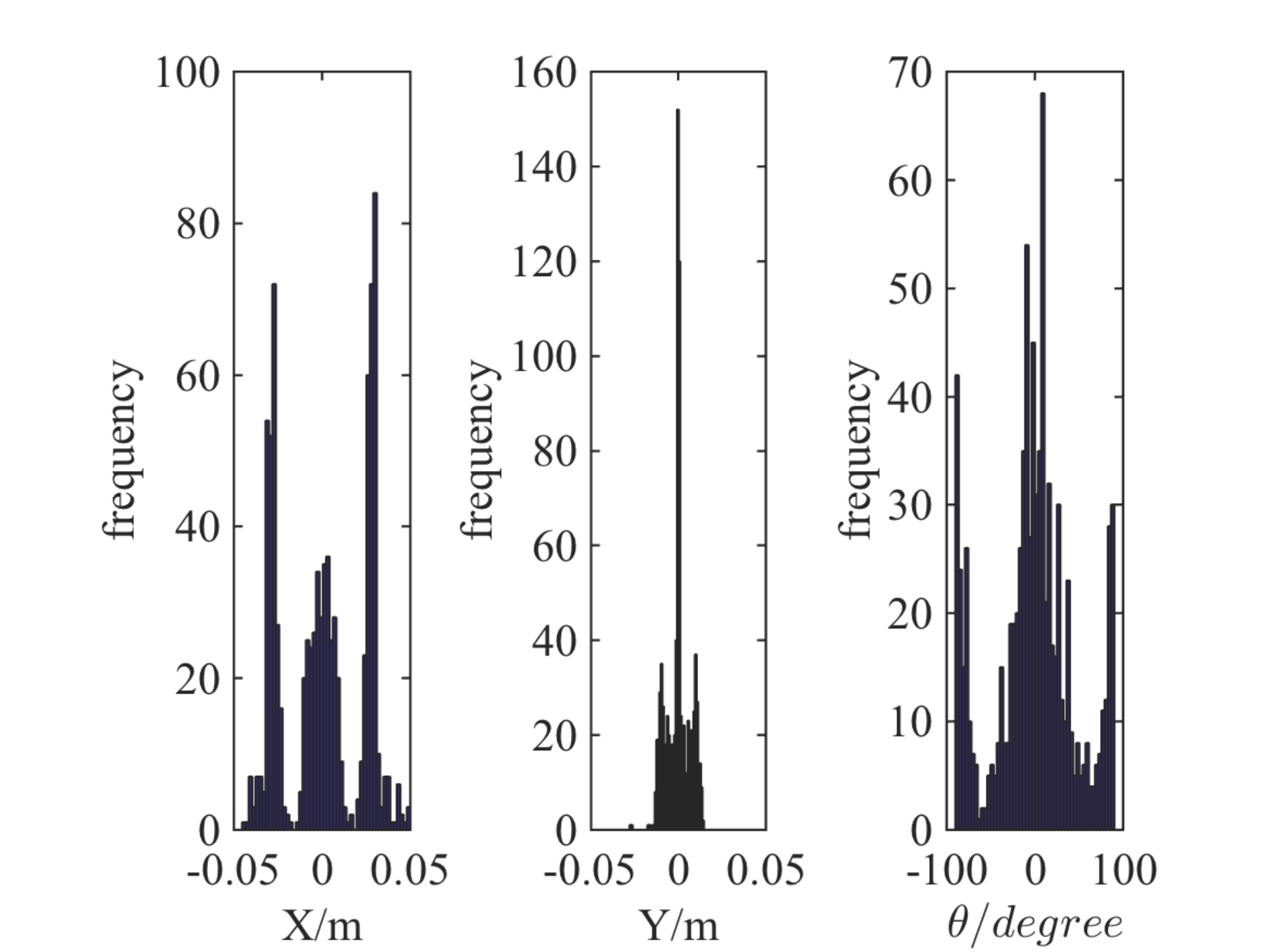}
\setcounter{subfigure}{7}%
\caption{Histogram plot of the experimental post distribution around
  the origin.}
\label{fig:hist_butter_zoom}
\end{subfigure}
\caption{Experiments on the butterfly object. The longer diameter from the convex curves is
  39mm and the shorter diameter from the concave curves is 28.6mm. 900 initial poses are
  sampled where the centers lie uniformly in a circle of radius 30mm
  and the frame angles are uniformly distributed in -90 to 90 degrees.}
\end{figure*}

\section{Conclusions and Future Work}
We extend the convex polynomial force-motion model in \cite{Zhou16}
which gives the dual mapping between friction wrench and twist to
the kinematic level where the applied controls are velocity input
(single and multiple) contacts. 
Additionally, we derive methods that enable sampling from the family
of sos-convex polynomials to model the inherent uncertainty in
frictional mechanics. The stochastic contact models are validated with
large scale robotic pushing and grasping experiments. We also see the
limitation of a first order quasistatic model in the butterfly shaped
object grasping experiment.  
Much work remains to be done. On the simulator end: 1) how to increase the accuracy without losing
convergence speed for high order polynomial based representation of
$H(\mathbf{F})$ and 2) how to handle penetration
properly when the integration step is large.  
On the application side: 1) how to quickly identify both the
mean and variance of the sampling distribution to match with
experimental data and 2) how to plan a robust sequence of grasp and push
actions for uncertainty reduction using the stochastic contact model. 

\section*{Acknowledgments}
The authors would like to thank Alberto Rodriguez and Kuan-Ting Yu for discussions and providing details of the MIT pushing dataset. This work was conducted in part through collaborative participation in the Robotics Consortium sponsored by the U.S Army Research Laboratory under the Collaborative Technology Alliance Program, Cooperative Agreement W911NF-10-2-0016 and National Science Foundation IIS-1409003. The views and conclusions contained in this document are those of the authors and should not be interpreted as representing the official policies, either expressed or implied, of the Army Research Laboratory of the U.S. Government. The U.S. Government is authorized to reproduce and distribute reprints for Government purposes notwithstanding any copyright notation herein.

\bibliographystyle{plainnat}
\bibliography{ref}

\end{document}